\ifdefined\pdftexversion\pdfoutput=1\fi
\documentclass[11pt]{article}

\usepackage[utf8]{inputenc}
\usepackage[T1]{fontenc}
\usepackage{booktabs}
\usepackage{wrapfig}
\usepackage{longtable}
\usepackage{xurl}
\usepackage{amsfonts}
\usepackage{nicefrac}
\usepackage{microtype}
\usepackage[table]{xcolor}
\PassOptionsToPackage{hyphens}{url}

\usepackage{smile}

\usepackage{tikz}
\usetikzlibrary{matrix, positioning}
\usepackage{pgfplots}
\pgfplotsset{compat=1.18}
\usepackage{wrapfig}
\usepackage{needspace}
\usepackage{xcolor}
\usepackage{cancel}
\usepackage{acronym}

\definecolor{mypink}{HTML}{FF1493}
\definecolor{darkred}{HTML}{BB253B}
\definecolor{lightgreen}{HTML}{90C0BF}

\definecolor{darkcyan}{HTML}{0081B9}

\newmdenv[
  backgroundcolor=gray!10,
  skipabove=1em,
  skipbelow=1em,
  leftline=false,
  topline=false,
  bottomline=false,
  rightline=false,
  linecolor=gray!88,
  linewidth=4pt
]{githubquote}

\usepackage[most,skins,theorems]{tcolorbox}
\tcbset{
  aibox/.style={
    width=\linewidth,
    top=10pt,
    bottom=4pt,
    colback=blue!6!white,
    colframe=black,
    colbacktitle=black,
    enhanced,
    center,
    attach boxed title to top left={yshift=-0.1in,xshift=0.15in},
    boxed title style={boxrule=0pt,colframe=white,},
  }
}

\definecolor{rliableolive}{HTML}{BBCC33}
\definecolor{rliableblue}{HTML}{77AADD}
\definecolor{rliablered}{HTML}{EE8866}
\definecolor{LightCyan}{rgb}{0.88,1,1}
\definecolor{darkblue}{HTML}{2878D9}
\definecolor{navyblue}{HTML}{0000FF}

\newtcolorbox{AIbox}[2][]{aibox,title=#2,colback=rliableblue!10!white,#1}

\providecolor{yaleblue}{HTML}{00356B}
\colorlet{findingframe}{yaleblue}
\definecolor{findingaccent}{HTML}{D97757}
\colorlet{findingbg}{yaleblue!5!white}

\newtcolorbox[%
  auto counter,
  number within=section,
  crefname={finding}{findings},
  Crefname={Finding}{Findings},
]{finding}[1][]{%
  enhanced,
  breakable,
  before skip=10pt,
  after skip=10pt,
  colback=findingbg,
  colframe=findingframe,
  boxrule=0pt,
  arc=3pt,
  outer arc=3pt,
  left=12pt, right=12pt, top=8pt, bottom=10pt,
  before upper={%
    \noindent\textbf{Finding~\thetcbcounter\ifblank{#1}{}{. #1}.}\enspace
  },
}

\usepackage{subcaption}

\usepackage{titlesec}
\usepackage{xcolor}
\definecolor{sectionblue}{HTML}{00356B}

\titleformat{\section}
  {\Large\bfseries\color{sectionblue}}{\thesection}{1em}{}
\titleformat{\subsection}
  {\large\bfseries\color{sectionblue}}{\thesubsection}{1em}{}
\titleformat{\subsubsection}
  {\normalsize\bfseries\color{sectionblue}}{\thesubsubsection}{1em}{}
\titleformat{\paragraph}[runin]
  {\normalfont\normalsize\bfseries\color{black}}{}{0pt}{}

\usepackage{paperstyle}

\acrodef{BPTT}{backpropagation through time}

\hypersetup{hypertexnames=false}

\graphicspath{{figs/}}
\definecolor{checkgreen}{HTML}{1B7F3B}
\definecolor{crossred}{HTML}{C0392B}
\newcommand{\yesmark}{\textcolor{checkgreen}{\ding{51}}}
\newcommand{\nomark}{\textcolor{crossred}{\ding{55}}}

\title{Truncate or Not? A Scaling Study of Looped Transformers with Truncated Backpropagation}
\papershorttitle{Truncate or Not?}
\author{Xuyuan~Xiong$^{1,*}$, Siyu~Chen$^{2,*}$, Zihan~Shao$^{2}$, Yue~Su$^{1}$, Weiyang~Jin$^{1}$, Ping~Luo$^{1}$, Zhuoran~Yang$^{2,\dagger}$, Xihui~Liu$^{1,\dagger}$}
\paperaffiliations{$^{1}$MMLab, The University of Hong Kong\qquad $^{2}$Yale University}
\papernote{$^{*}$Equal contribution: \paperemail{xiongxuyuan@connect.hku.hk}, \paperemail{siyu.chen.sc3226@yale.edu}\\%
  $^{\dagger}$Corresponding authors: \paperemail{zhuoran.yang@yale.edu}, \paperemail{xihuiliu@eee.hku.hk}}
\paperlogos{\paperlogoitem{0.68cm}{logos/hku_mmlab}\hspace{0.5cm}\paperlogoitem{1.12cm}{logos/hku}\hspace{0.5cm}\paperlogoitem{1.08cm}{logos/yale}}
\papermeta{Date}{October 15, 2026}
\papermeta{Code}{\url{https://github.com/xjxyys/Truncate-or-Not}}

\begin{document}

\maketitle
\pagestyle{paper}

\begin{abstract}
Looped transformers reuse a shared block of layers to increase computational depth without increasing the number of trainable parameters.
They are trained with \ac{BPTT}. To save training compute, \ac{BPTT} is often truncated to the last few loops,
but truncation yields biased gradients that may weaken the benefit of looping.
In this work, we study the value of truncation in looped transformer training in a systematic way: we jointly vary model size, training tokens, and
forward and backward loop counts across nearly 600 training runs and fit a unified scaling law that captures the effects of
looping and truncation through \emph{effective parameters}. Despite the training-compute savings of truncation, our fitted law favors
\textbf{full \ac{BPTT}} when forward and backward loop counts are jointly optimized. In matched-compute runs, full \ac{BPTT}
through fewer loops consistently outperforms truncated settings of the same compute. The law further shows that at fixed model
size, the compute-optimal loop count and training-token budget both grow with compute.
When model size is also free, the law identifies two fully trained loops as the best looped setting.
Experimentally, this setting achieves lower validation loss than the compute-optimal non-looped model with about
\textbf{27\% fewer non-embedding parameters} and only 8\% more training compute.
\end{abstract}

\begin{figure}[!h]
  \vspace{-4pt}
  \centering
  \captionsetup{font=small, skip=3pt}
  \includegraphics[width=\textwidth]{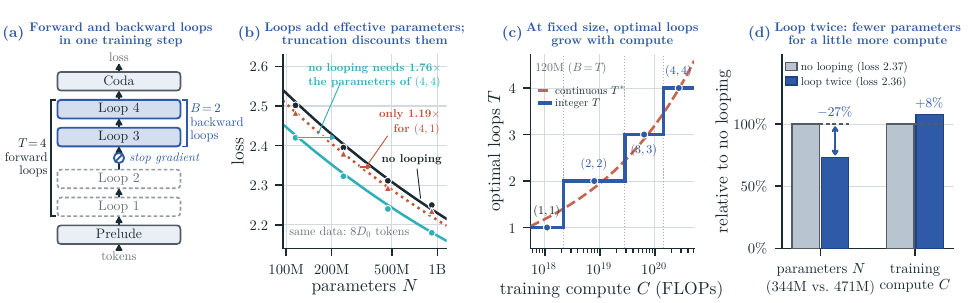}
  \caption{\textbf{Overview of looped transformer scaling with truncated \ac{BPTT}.}
  \textbf{(a)}~Step for $(T,B)=(4,2)$ ($T$ forward loops, last $B$ trained).
  \textbf{(b)}~At $8D_0$ tokens, loops add effective parameters while truncation discounts them (\Cref{sec:law}; lines: law, markers: runs).
  \textbf{(c)}~At fixed size (120M), optimal $T$ grows with compute under full \ac{BPTT} (\Cref{sec:law-fixed-size}).
  \textbf{(d)}~$(2,2)$ (344M) beats $(1,1)$ (471M, $2.1\times10^{19}$ FLOPs) in loss ($2.36$ vs.\ $2.37$ nats) with $27\%$ fewer parameters and $8\%$ more compute (\Cref{sec:untie}).}
  \label{fig:teaser}
  \vspace{-6pt}
\end{figure}

\section{Introduction}
\label{sec:intro}

Looped transformers have attracted attention for combining strong performance with efficient use of parameters and training
memory. By repeatedly applying a shared block of layers, they increase computational depth without increasing the number of
trainable parameters \citep{dehghani2019universal,geiping2025scaling}. To reduce the cost of training these repeated blocks,
several looped language models use truncated \acf{BPTT}, propagating gradients only through the final loops while executing
the earlier loops without tracking gradients \citep{geiping2025scaling,mcleish2025teaching,prairie2026parcae,wang2026hrmtext,huang2026towards}.
Truncation introduces a new tradeoff: it saves backward computation and activation memory during training but can weaken
the gains from looping \citep{schwethelm2026how}. To achieve the best performance under a fixed training compute
budget, we must therefore decide how to divide compute among forward loops, backward loops, model size,
and training tokens.

Existing scaling studies examine prescribed backpropagation recipes rather than jointly modeling forward and backward loop
counts. \citet{schwethelm2026how} compare full \ac{BPTT} with truncation to half the loops, fitting a separate scaling
law to each recipe. Parcae \citep{prairie2026parcae} varies backward-loop counts in small-model experiments but fixes them to half
the forward-loop count in its scaling study, leaving their joint optimization to future work. These studies characterize specific
recipes but do not provide a unified law for choosing both loop counts across model sizes and token budgets
(Table~\ref{tab:looped_studies}). We thus ask the following question: 

\vspace{0.5\baselineskip}
{\centering
\begin{minipage}{0.9\linewidth}
\centering\itshape
How do model size, training tokens, and forward and backward loops jointly affect loss,
and should we truncate backpropagation under a fixed training compute budget?
\end{minipage}
\par}
\vspace{0.5\baselineskip}

To characterize how these quantities affect loss, we conduct nearly {\bf 600 training runs}, including a scaling suite spanning
58M--919M non-embedding parameters, over three orders of magnitude in training tokens per parameter, and up to 16 forward loops.
We independently vary forward and backward loop counts across model sizes and token budgets and fit a {\bf unified scaling law}.
The scaling law expresses the benefit of looping through an \emph{effective parameter count}---the size
of a non-looped model that achieves the same loss at the same token budget. Fully trained loops increase the effective
parameter count with diminishing returns, while truncation discounts this gain.

We then use the scaling law to decide whether to truncate, and the answer depends on whether the forward-loop count is prescribed.
When it is, the law recovers the benefit of truncation reported in prior work \citep{schwethelm2026how}: the saved backward
compute supports a larger model and more training tokens, which more than offset the discounted loop gain.
When both loop counts can be chosen, however, the law favors full \ac{BPTT}. Reducing the forward-loop count frees
compute for additional backward loops until every remaining loop is trained, and under the fitted law this reallocation yields
more effective parameters at the same compute than the original truncated setting. Separately fitted laws and
matched-compute experiments support this advantage (\Cref{sec:truncate}).

Building on this preference for full \ac{BPTT}, we additionally study how to allocate compute among loop count, training tokens,
and model size. Consider first a setting where model size is predetermined, for example due to device memory constraints.
At fixed model size, increasing the loop count is an effective way to increase the effective parameter count;
however, additional loops offer diminishing gains while increasing the cost per token. The law precisely
characterizes the optimal tradeoff between loop count and training tokens, allocating most additional compute to tokens
and increasing the loop count only slowly. This result extends Parcae's fixed-size allocation analysis
\citep{prairie2026parcae} beyond a prescribed truncation rule. When model size is also free,
compute can instead support more unique parameters; under our leading-FLOP approximation, the law favors a non-looped model
and identifies two fully trained loops as the best looped setting (Figure~\ref{fig:law-free-size}). In experiments, models with two fully trained loops nearly
match compute-optimal non-looped models at equal training compute while having substantially fewer non-embedding parameters.
With a little more training compute, such a smaller model surpasses the compute-optimal non-looped model
(\Cref{sec:compute}).

\begin{figure}[t]
  \centering
  \captionsetup{font=small, skip=3pt}
  \includegraphics[width=0.95\textwidth]{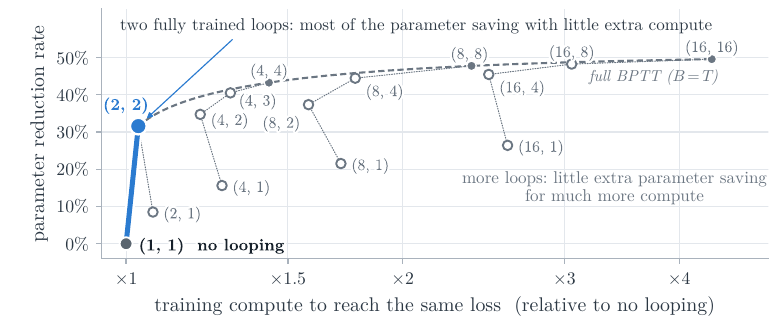}
  \caption{\textbf{No looping is computationally optimal; two fully trained loops saves most parameters with little extra compute.} Each marker is a loop setting $(T,B)$, with $T$ forward loops
  and the last $B$ trained, at the parameter and token counts our scaling law finds compute-optimal. Compared with the model
  without looping at the same loss, the horizontal axis shows the training compute it needs and the vertical axis the
  fraction of non-embedding parameters it saves. Filled markers use full BPTT ($B=T$), and the dashed curve is the law's prediction for full BPTT at any $T$ from 2 to 16. Every truncated setting lies below this curve, so full BPTT forms the frontier and saves the most parameters at each compute. Dotted lines join the settings with the same $T$ (\Cref{sec:law-free-size}).}
  \label{fig:law-free-size}
\end{figure}

Our contributions are threefold:
\begin{itemize}[leftmargin=1.6em,itemsep=2pt,topsep=3pt]
\item \textbf{A unified scaling law based on effective parameters.} We develop and validate a law that expresses loop gains
and truncation discounts through an effective parameter count, jointly modeling model size, training tokens, and independent
forward and backward loop counts (\Cref{sec:law}).
\item \textbf{An explanation of when truncation pays.} We show why truncation can pay at a prescribed forward-loop count and why
joint forward--backward allocation favors full \ac{BPTT} under the fitted law, supported by matched-compute experiments
(\Cref{sec:truncate}).
\item \textbf{Compute-allocation rules and parameter-efficient training.} We derive allocation rules for fixed and free model
size and demonstrate that two fully trained loops nearly match compute-optimal non-looped models while having substantially
fewer trainable parameters (\Cref{sec:compute}).
\end{itemize}
Figure~\ref{fig:teaser} previews the three results: loops add effective parameters and truncation discounts them; at
fixed model size the best loop count grows with compute; two fully trained loops match the loss of a non-looped model
with fewer parameters for a little more compute.

\section{Related work}
\label{sec:related}

\paragraph{Looped transformers.}
Depth-wise parameter sharing in Universal Transformers and ALBERT decouples computational depth from the number of trainable parameters \citep{dehghani2019universal,lan2020albert}.
Theoretical and controlled studies show how recurrence supports iterative programs, in-context learning algorithms, and depth-sensitive reasoning \citep{giannou2023looped,yang2024looped,saunshi2025reasoning}.
In language model pretraining, Huginn uses a prelude--recurrent--coda architecture \citep{geiping2025scaling}, while Ouro scales looped pretraining to 7.7T tokens \citep{zhu2025scaling}.
Other work adds recurrence to pretrained models \citep{bae2025relaxed,mcleish2025teaching} or adapts depth at inference \citep{bae2025mixture,jeddi2026loopformer}.
Recent architectures improve recurrence through residual and hyper-connections \citep{prairie2026parcae,fu2026simply,marchenko2026closing,zeitoun2026hyperloop} or MoE routing \citep{csordas2024moeut,chen2026loopmoe,he2026looping,gao2026loop}.
These studies show that looping improves parameter efficiency across architectures; we instead fix the architecture and ask how to train its loops and allocate compute.

\paragraph{Training through recurrent depth.}
Truncated \ac{BPTT} limits gradient propagation to a subset of recurrent steps, saving activation memory and backward compute
at the cost of biased gradients \citep{williams1990efficient,tallec2017unbiasing}.
It is the common recipe for looped language models, which backpropagate through the last few loops or half of them
\citep{geiping2025scaling,mcleish2025teaching,prairie2026parcae,schwethelm2026how}, and recursive reasoning models often keep
only the last step \citep{wang2025hierarchical,jolicoeur2025trm}. \Cref{app:related} surveys these recipes in more detail.
Implicit models likewise differentiate only a few steps at the fixed point \citep{bai2019deep,geng2021training}, and truncated
\ac{BPTT} in looped transformers can be read as a truncated Neumann series of the fixed-point gradient
\citep{movahedi2026fixedpoint,huang2026towards}.
The few studies that measure what truncation costs compare training recipes at a fixed forward depth
\citep{feinashley2026attractor,huang2026eqr,salvatori2026trmp,ren2026hrm}. We instead vary the backward depth independently of
the forward depth, weighing the compute that truncation saves against the loop gain it costs.

\paragraph{Scaling laws for recurrence.}
Size--data scaling laws quantify how parameters and training tokens jointly determine loss
\citep{kaplan2020scaling,hoffmann2022training,videau2026skaling}.
For looped models, some laws fit each architecture separately \citep{wang2026smelt,lee2026sparse,chen2026growth}, while others
model the loop count explicitly (Table~\ref{tab:looped_studies}) through additive loss terms \citep{zhu2025scaling} or effective
parameters with linear, power-law, or bounded gains \citep{prairie2026parcae,schwethelm2026how,chen2026scaling}.
Closest to our work, Iso-Depth \citep{schwethelm2026how} fits separate laws for full and half-length backpropagation and finds
that truncation weakens the loop gain but lowers the loss at a fixed forward-loop count by freeing compute, while Parcae
\citep{prairie2026parcae} finds at fixed model size that loops and tokens should grow together with compute but ties the
backward loops to half the forward loops and leaves their joint choice open.
Our law instead treats the backward-loop count as a separate variable, combining coupled size--data scaling
\citep{videau2026skaling} with bounded loop gains \citep{chen2026scaling}. It recovers the truncation benefit at a prescribed
forward-loop count but favors full \ac{BPTT} when both counts are free (\Cref{sec:truncate}).
Looped MoE models commonly use two loops \citep{he2026looping,wang2026smelt,gao2026loop}. For dense models, our law ranks two
fully trained loops first among looped settings, since they provide nearly the effective size of two untied copies
(\Cref{sec:law-free-size}), and in our runs they nearly match compute-optimal models without looping with substantially fewer
non-embedding parameters (\Cref{sec:untie}).

\begin{table}[t]
  \centering
  \caption{Scaling laws for looped language models. Earlier laws either fit a separate curve per architecture or model looping through a single loop count $T$, with backward loops $B$ fixed or tied to $T$; ours models $T$ and $B$ independently with a bounded effective-parameter gain. /: not discussed; $^\dagger$appendix scaling fit.}
  \label{tab:looped_studies}
  \footnotesize
  \setlength{\tabcolsep}{3.6pt}
  \begin{tabular*}{\textwidth}{@{\extracolsep{\fill}}ll *{4}{>{\centering\arraybackslash}p{0.55cm}} lll@{}}
    \toprule
    & & \multicolumn{4}{c}{Scaling axes} & & & \\
    \cmidrule(lr){3-6}
    Study & Arch. & $N$ & $D$ & $T$ & $B$ & Loop gain & BPTT & Regime \\
    \midrule
    \citet{schwethelm2026how} & dense & \yesmark & \yesmark & \yesmark & \nomark & power law ($T^\phi$) & $T$; $\lceil T/2\rceil$ & fixed depth \\
    Parcae \citep{prairie2026parcae} & dense & \yesmark & \yesmark & \yesmark & \nomark & linear ($T$) & $\lceil T/2\rceil$ & fixed $N$ \\
    \citet{chen2026scaling} & dense, MoE & \yesmark & \yesmark & \yesmark & \nomark & bounded exp.\ ($T$) & full ($T$) & free $N$, mem. \\
    SMELT \citep{wang2026smelt} & MoE & \yesmark & \yesmark & \nomark & \nomark & per-arch.\ fit & / & matched $N$, KV \\
    \citet{lee2026sparse} & dense, MoE & \yesmark & \nomark & \nomark & \nomark & per-arch.\ fit & / & free $N$ \\
    \citet{chen2026growth} & dense & \yesmark & \yesmark & \nomark & \nomark & per-arch.\ fit & full ($T$) & free $N$ \\
    Ouro$^\dagger$ \citep{zhu2025scaling} & dense & \yesmark & \yesmark & \yesmark & \nomark & additive ($T$) & / & / \\
    \midrule
    Ours & dense & \yesmark & \yesmark & \yesmark & \yesmark & bounded $(T,B)$ & $1\le B\le T$ & fixed \& free $N$ \\
    \bottomrule
  \end{tabular*}
\end{table}

\section{Preliminaries}
\label{sec:prelim}

\begin{figure}[t]
  \centering
  \includegraphics{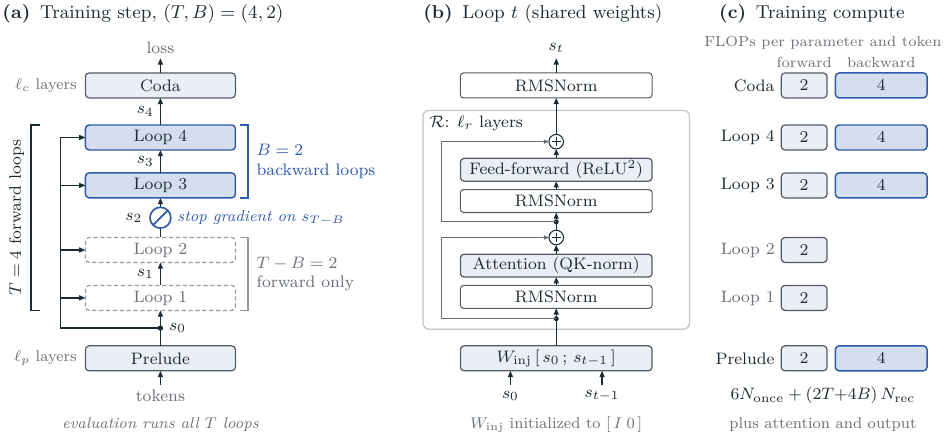}
  \caption{\textbf{The looped transformer, drawn for $(T,B)=(4,2)$.} \textbf{(a)}~One training step: a stop gradient on
  $s_{T-B}$, so the first $T-B$ loops (dashed) run forward only. \textbf{(b)}~One loop, \eqref{eq:loop}; all loops share
  its weights. \textbf{(c)}~Training FLOPs per parameter and token of each part \eqref{eq:law-leading-compute}.}
  \label{fig:arch}
\end{figure}

\paragraph{Looped transformers.}
We study the prelude--recurrent--coda architecture \citep{geiping2025scaling,schwethelm2026how}, illustrated in Figure~\ref{fig:arch} and detailed in \Cref{app:arch}.
A prelude $\mathcal{P}$ of $\ell_p$ layers, including the input embedding, maps the input tokens to hidden states.
A shared recurrent block $\mathcal{R}$ of $\ell_r$ layers then refines these states over $T$ recurrent iterations, which we call \emph{loops}, and a coda $\mathcal{C}$ of $\ell_c$ layers, including the output head, turns the result into next-token predictions.
For an input token sequence $x$ and loops $t=1,\dots,T$, the model computes
\begingroup
\addtolength{\abovedisplayskip}{14pt}%
\addtolength{\abovedisplayshortskip}{14pt}%
\begin{equation}
    s_0 = \tikz[baseline=(prelude.base)]{%
      \node[inner sep=0pt, outer sep=0pt] (prelude) {$\mathcal{P}$};
      \node[overlay, anchor=base, inner xsep=3pt, inner ysep=1.5pt,
        text height=6pt, text depth=1.5pt, rounded corners=1pt,
        fill=pageaccent!8, font=\sffamily\footnotesize, text=pageaccent!85]
        (preludelabel) at ([yshift=20pt]prelude.base) {prelude};
      \draw[overlay, ->, >=stealth, pageaccent!60, line width=0.35pt]
        ([yshift=15.5pt]prelude.base) -- ([yshift=1.5pt]prelude.north);
    }(x), \qquad
    s_t = \Phi_\theta(s_{t-1};s_0) = \mathrm{RMSNorm}\bigl(\tikz[baseline=(recurrent.base)]{%
      \node[inner sep=0pt, outer sep=0pt] (recurrent) {$\mathcal{R}$};
      \node[overlay, anchor=base, inner xsep=3pt, inner ysep=1.5pt,
        text height=6pt, text depth=1.5pt, rounded corners=1pt,
        fill=pageaccent!8, font=\sffamily\footnotesize, text=pageaccent!85]
        (recurrentlabel) at ([xshift=-18pt,yshift=20pt]recurrent.base) {loop block};
      \draw[overlay, ->, >=stealth, pageaccent!60, line width=0.35pt]
        (recurrentlabel.south) to[out=-90,in=120] ([yshift=1.5pt]recurrent.north);
    }\bigl(W_{\mathrm{inj}}\,[\,%
      \tikz[baseline=(loopinput.base)]{%
        \node[inner sep=0pt, outer sep=0pt] (loopinput) {$s_0$};
        \node[overlay, anchor=base, inner xsep=3pt, inner ysep=1.5pt,
        text height=6pt, text depth=1.5pt, rounded corners=1pt,
        fill=pageaccent!8, font=\sffamily\footnotesize, text=pageaccent!85]
          (looplabel) at ([xshift=20pt,yshift=20pt]loopinput.base) {input injection};
        \draw[overlay, ->, >=stealth, pageaccent!60, line width=0.35pt]
          (looplabel.south) to[out=-90,in=60] ([yshift=1.5pt]loopinput.north);
      }\,;\,s_{t-1}\,]\bigr)\bigr), \qquad
    p = \tikz[baseline=(coda.base)]{%
      \node[inner sep=0pt, outer sep=0pt] (coda) {$\mathcal{C}$};
      \node[overlay, anchor=base, inner xsep=3pt, inner ysep=1.5pt,
        text height=6pt, text depth=1.5pt, rounded corners=1pt,
        fill=pageaccent!8, font=\sffamily\footnotesize, text=pageaccent!85]
        (codalabel) at ([yshift=20pt]coda.base) {coda};
      \draw[overlay, ->, >=stealth, pageaccent!60, line width=0.35pt]
        ([yshift=15.5pt]coda.base) -- ([yshift=1.5pt]coda.north);
    }(s_T).
  \label{eq:loop}
\end{equation}
\endgroup
Here $s_0,s_1,\dots,s_T\in\mathbb{R}^{d}$ are the hidden states at each token position: the prelude output $s_0$ is the initial state, and $s_t$ is the state after loop $t$.
Each loop concatenates $s_0$ with the previous state into $[\,s_0\,;\,s_{t-1}\,]$, projects the result back to width $d$ with $W_{\mathrm{inj}}\in\mathbb{R}^{d\times 2d}$, and passes it through $\mathcal{R}$ and an RMSNorm to obtain $s_t$.
This \emph{input injection} feeds $s_0$ to every loop, and the RMSNorm at the end of each loop keeps the state passed to the next loop or the coda at a controlled scale, as in \citet{schwethelm2026how}.
We denote the loop map by $\Phi_\theta$, where $\theta$ collects the recurrent weights of $\mathcal{R}$ and $W_{\mathrm{inj}}$.
The coda maps the final state $s_T$ to the next-token distribution $p$, and the cross-entropy loss applies only to $p$.
We split the non-embedding parameters as $N=N_{\mathrm{once}}+N_{\mathrm{rec}}$, where $N_{\mathrm{once}}$ counts $\mathcal{P}$ and $\mathcal{C}$, which run once per token, and $N_{\mathrm{rec}}$ counts $\mathcal{R}$ and $W_{\mathrm{inj}}$, which are reused in every loop; Table~\ref{tab:models} gives the share $N_{\mathrm{rec}}/N$ for each model size.

\paragraph{Truncated BPTT.}
A \emph{loop setting} $(T,B)$ with $1\le B\le T$ runs $T$ loops forward and backpropagates through only the last $B$: the state $s_{T-B}$ is detached (Figure~\ref{fig:arch}a), so the first $T-B$ loops carry no gradient and store no activations.
The setting $B=T$ is full \ac{BPTT}, $B<T$ is truncated \ac{BPTT}, and $(1,1)$ is a standard transformer.
The recurrent weights $\theta$ and, through the input injection, the prelude receive only the gradient of the last $B$ loops. The coda's gradient is exact.
Evaluation always runs all $T$ loops.
A backward loop costs twice the FLOPs of a forward loop (\Cref{app:arch}), so truncation trades backward compute and activation memory for a biased gradient \citep{williams1990efficient,tallec2017unbiasing}. \Cref{app:bptt} writes out both gradients and discusses this bias.

\paragraph{Baseline scaling law.}
To relate validation loss to non-embedding parameter count $N$ and training tokens $D$ for a non-looped model, we adopt the coupled scaling law of \citet{videau2026skaling},
\begin{equation}
  L(N,D) \;=\; E + \Bigl[A\cdot\Bigl(\frac{N}{N_0}\Bigr)^{-\alpha}
    + H\cdot\Bigl(\frac{D}{D_0}\Bigr)^{-\beta}\Bigr]^{\kappa}. 
  \label{eq:coupled-scaling}
\end{equation}
Here $E$ is the irreducible loss, $N_0$ and $D_0$ are fixed reference scales (we give token budgets in units of $D_0=2^{30}$ tokens), $\alpha,\beta>0$ are the size and token exponents, and the outer exponent $\kappa$ couples the two terms ($\kappa=1$ recovers the additive form of \citet{hoffmann2022training}); we denote the data coefficient by $H$ to reserve $B$ for backward loops.
\Cref{sec:law} extends \eqref{eq:coupled-scaling} to looped models by replacing $N$ with an effective parameter count in which $(T,B)$ rescales $N_{\mathrm{rec}}$ by a \emph{looping coefficient} $g(T,B)$.

\section{A scaling law for forward and backward loops}
\label{sec:law}

Should a looped transformer truncate backpropagation? At a fixed model size
and training compute, truncation saves backward computation that can be
spent on more forward loops or more training data, but it also reduces the
benefit of looping. We fit a scaling law that separates forward and
backward loop counts to quantify this tradeoff. This section fits and
validates the law and interprets it in terms of effective parameters.
\Cref{sec:truncate} then uses the law to decide whether to truncate, and
\Cref{sec:compute} uses it to divide compute among loops, data, and model size.

\subsection{Fitting and validating the scaling law}

\paragraph{Experimental setup.}
\label{sec:setup}
We follow the architecture and training recipe of \citet{schwethelm2026how}, using the
prelude--recurrent--coda architecture introduced in \Cref{sec:prelim}. We train a grid of five model sizes from 58M to 919M
non-embedding parameters (Table~\ref{tab:models}) on the 100B-token FineWeb-Edu corpus \citep{penedo2024fineweb}.
We use MuonH for transformer-block and input-injection matrices and AdamW for embedding, output, and normalization parameters
\citep{jordan2024muon,wen2025hyperball}. We report validation cross-entropy in nats per token on deduplicated
held-out shards; \Cref{app:setup} gives the full training and evaluation setup, including the learning-rate and batch-size rules.

\begin{table}[!htbp]
\centering
\begin{minipage}[t]{0.42\textwidth}
\vspace{0pt}
\centering
\caption{\textbf{Model sizes.}}
\label{tab:models}
\small
\setlength{\tabcolsep}{0pt}
\begin{tabular*}{\linewidth}{@{\extracolsep{\fill}}lrcrr@{}}
\toprule
Size & $d$ & Layers & $N$ & $N_{\mathrm{rec}}/N$ \\
\midrule
60M  & 768  & 2/4/2 & 57.8M & 0.510 \\
120M & 1024 & 2/5/2 & 115.4M & 0.564 \\
240M & 1280 & 3/6/3 & 239.2M & 0.507 \\
480M & 1664 & 3/7/4 & 470.8M & 0.506 \\
1B   & 2176 & 4/8/4 & 918.7M & 0.505 \\
\bottomrule
\end{tabular*}
\par\smallskip
{\footnotesize\raggedright
$d$: width; layers: $\ell_p / \ell_r / \ell_c$; $N$: non-embedding parameters (\Cref{app:setup}).
\par}
\end{minipage}\hfill
\begin{minipage}[t]{0.55\textwidth}
\vspace{0pt}
\centering
\caption{\textbf{Fit and prediction errors} ($\times10^{-3}$).}
\label{tab:law-fit-summary}
\small
\setlength{\tabcolsep}{0pt}
\begin{tabular*}{\linewidth}{@{\extracolsep{\fill}}lrrr@{}}
\toprule
Evaluation & Fits & RMSE & MAE \\
\midrule
\csname @@input\endcsname{sections/tab_exponential_main_rows.tex}
\bottomrule
\end{tabular*}
\par\smallskip
{\footnotesize\raggedright
Errors in $\times10^{-3}$ nats/token over 230 configurations ($R^2=0.9992$; \Cref{app:law}).
\par}
\end{minipage}
\end{table}
\begin{figure}[t]
  \centering
  \includegraphics[width=\textwidth]{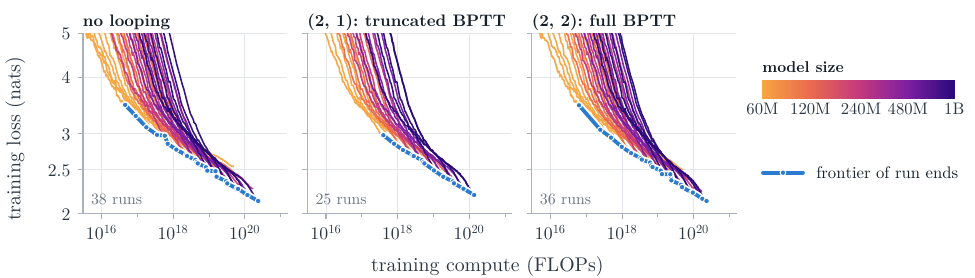}
  \caption{\textbf{Training curves of the scaling suite for no looping, $(2,1)$, and $(2,2)$.} Each curve is one run, colored by
  model size, against training compute (\Cref{app:compute-flops}); each token budget is a separate run. Training loss of every
  tenth step, averaged over at most $\pm250$ steps, not the validation loss we report. Highlighted line and dots: the frontier of
  run ends, the runs that end below every run with less compute. \Cref{app:curves} shows all fifteen loop settings.}
  \label{fig:curves}
\end{figure}

Our 592 training runs include a \emph{scaling suite} of 347 distinct configurations spanning 1.1--1{,}291 tokens per parameter
and fifteen $(T,B)$ combinations with $T\in\{1,2,4,8,16\}$, keeping each model's parameter count fixed \citep{prairie2026parcae}.
Figure~\ref{fig:curves} shows the training curves of three of these settings: as compute grows, larger models take over the
frontier of run ends, as in \citet{hoffmann2022training}.
We fit the law on 230 configurations
and reserve 117 for separate tests. For configurations with multiple accepted runs, we use their mean validation loss.
Each configuration receives the same weight in the fit, regardless of its number of runs.
Other experiments cover compute allocation, comparisons with untied models, learning-rate sweeps, and batch-size sweeps;
\Cref{app:setup,app:law-fit,app:untie,app:lrbs} give the full experimental design and fitting details.

\paragraph{Effective-parameter scaling law.}
Only the recurrent block runs in every loop, so looping can enlarge only this block. We therefore write the loss as the coupled law
of \citet{videau2026skaling} with an effective parameter count. In it, the recurrent block, $N_{\mathrm{rec}}$, counts several
times, and the prelude and coda, $N_{\mathrm{once}}$, count once (\Cref{sec:prelim}):
\begin{equation}
\begin{gathered}
  L = E + \Bigl[A\cdot\Bigl(\frac{N_{\mathrm{eff}}}{N_0}\Bigr)^{-\alpha} + H\cdot\Bigl(\frac{D}{D_0}\Bigr)^{-\beta}\Bigr]^{\kappa},\\
  N_{\mathrm{eff}} = N_{\mathrm{once}} + g(T,B)\cdot N_{\mathrm{rec}},
  \qquad g(T,B) = \exp\bigl[c\cdot(1-T^{-p}\cdot B^{-q})\bigr].
\end{gathered}
\label{eq:law}
\end{equation}
Here $g(T,B)$ is the \emph{looping coefficient}: it multiplies the recurrent block's parameter count, while the prelude and coda count once.
With full \ac{BPTT}, the looping coefficient $g(T,T)=\exp[c\cdot(1-T^{-(p+q)})]$ grows with $T$;
truncation multiplies it by $g(T,B)/g(T,T)\le1$.

\paragraph{Design of the looping coefficient.}
We choose $g(T,B)$ to satisfy four modeling requirements:
\definecolor{gainaccent}{HTML}{42566B}
\newcommand{\gainpanelheading}[2]{%
  {\pagesfsb\footnotesize\color{gainaccent}%
    \tikz[baseline=(gainnumber.base)]{\node[rounded corners=3pt,
      fill=gainaccent!10!white, text=gainaccent, inner sep=0pt,
      minimum width=13pt, minimum height=13pt, font=\pagesfsb\scriptsize]
      (gainnumber) {#1};}%
    \hspace{4pt}#2}\par\vspace{2pt}}
\newcommand{\gainpanelequation}[1]{%
  \noindent\parbox[c][28pt][c]{\linewidth}{\centering\small\color{pageink}\(#1\)}\par\vspace{2pt}}
\begin{tcbraster}[raster columns=4, raster equal height=all, raster column skip=6pt,
  raster before skip=4pt, raster after skip=8pt,
  enhanced, unbreakable, arc=8pt, outer arc=8pt, boxrule=0.45pt,
  colback=gainaccent!3!white, colframe=gainaccent!20!white,
  boxsep=0pt, left=6pt, right=6pt, top=7pt, bottom=7pt,
  fontupper=\footnotesize\color{pageink!80!white},
  before upper={\raggedright\setlength{\parskip}{0pt}}]
\begin{tcolorbox}
  \gainpanelheading{1}{Normalization}
  \gainpanelequation{\begin{gathered}g(T,B)\ge0,\\ g(1,1)=1\end{gathered}}
  The looping coefficient is nonnegative; a single loop recovers the non-looped model.
\end{tcolorbox}
\begin{tcolorbox}
  \gainpanelheading{2}{Monotonicity}
  \gainpanelequation{\partial_T g>0,\ \ \partial_B g>0}
  Increasing either loop count increases the looping coefficient.
\end{tcolorbox}
\begin{tcolorbox}
  \gainpanelheading{3}{Zero-backward limit}
  \gainpanelequation{g(T,0)=0}
  With no trained loop, the recurrent block adds no effective size.
\end{tcolorbox}
\begin{tcolorbox}
  \gainpanelheading{4}{Saturation}
  \gainpanelequation{\displaystyle\sup_{1\le B\le T}g(T,B)<\infty}
  Finite parameters store finite information, motivating a bounded gain.
\end{tcolorbox}
\end{tcbraster}
\par\medskip
\begin{table}[!htbp]
\centering
\small
\caption{Fitted parameters on the 230-configuration fitting set.
The reference scales are fixed at $N_0=2.4\times10^8$ parameters and
$D_0=2^{30}$ tokens.}
\label{tab:law-design-coefficients}
\setlength{\tabcolsep}{8.5pt}
\begin{tabular}{@{}rrrrrrrrr@{}}
\toprule
$E$ & $A$ & $H$ & $\alpha$ & $\beta$ & $\kappa$ & $c$ & $p$ & $q$ \\
\midrule
0.29865 & 6.48695 & 7.63343 & 0.30568 & 0.39281 & 0.32510 & 1.11198 & 0.23677 & 1.02181 \\
\bottomrule
\end{tabular}
\end{table}

The finite-ceiling assumption is supported by prior observations of saturation:
\citet{chen2026scaling} model a bounded effective-parameter gain from
looping, while Parcae finds that adding test-time loops to a trained model
approaches a loss floor \citep{prairie2026parcae}.
There is also an information-storage motivation. At a fixed parameter
count and weight precision, the weights can store only a finite amount
of information; additional loops provide more computation over those
weights without increasing their storage capacity. We therefore treat saturation as an empirical
modeling assumption and learn its ceiling from the data.

These requirements motivate the form in \eqref{eq:law}. Nonnegativity motivates an exponential parameterization,
ensuring $g(T,B)>0$ for $T,B>0$. With $c,p,q>0$, increasing either loop count
decreases $T^{-p}\cdot B^{-q}$, so $g(T,B)$ grows toward the finite ceiling
$\exp(c)$.
The shared coefficient $c$ enforces $g(1,1)=1$,
and continuous extension as $B\to0^+$ gives $g(T,0)=0$.
At fixed model size and token budget, $p$ governs the loss reduction from
extra forward loops at fixed $B$, while $q$ governs the loss reduction
from extra backward loops at fixed $T$. These three loop coefficients,
$c,p,q$, together with the six coefficients $E,A,H,\alpha,\beta,\kappa$
of the coupled size--data law \eqref{eq:coupled-scaling}, give
nine fitted coefficients.

\begin{figure}[!t]
\centering
\begin{minipage}[b]{0.43\textwidth}
  \centering
  \includegraphics[width=\linewidth]{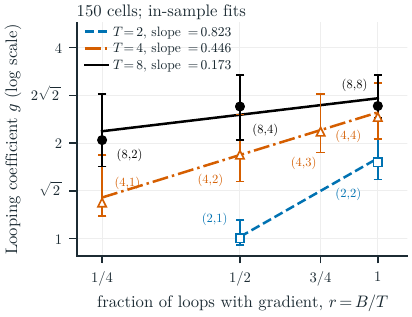}
\end{minipage}\hfill
\begin{minipage}[b]{0.55\textwidth}
  \centering
  \includegraphics[width=\linewidth]{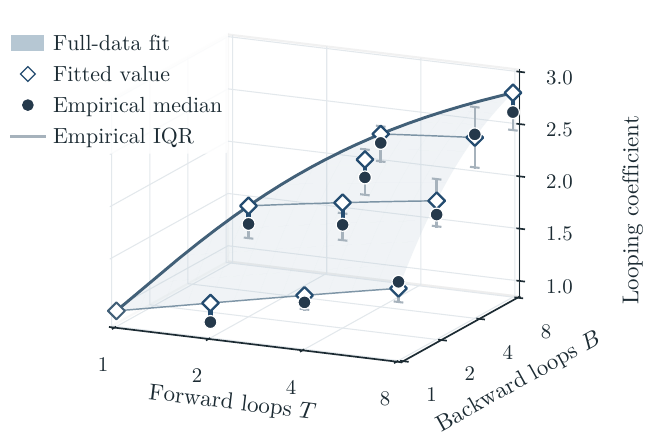}
\end{minipage}
\par
\begin{minipage}[t]{0.43\textwidth}
  \captionsetup{width=\linewidth}
  \caption{\textbf{The looping coefficient varies with $r=B/T$.}
  Markers and bars: medians and ranges of the loss-matched empirical
  coefficient $\widehat g$ over 150 configurations. Lines: least-squares fits
  of $\log_2\widehat g$ against $\log_2 r$ for each $T$ (\Cref{app:law-size-estimates}). The slope falls as $T$ grows:
  truncating the same fraction of loops costs less at larger $T$, so $g$
  couples $T$ and $B$.}
  \label{fig:backward-nats}
\end{minipage}\hfill
\begin{minipage}[t]{0.55\textwidth}
  \captionsetup{width=\linewidth}
  \caption{\textbf{The looping coefficient and the runs.}
  Surface and diamonds: the full-data fitted looping coefficient $g(T,B)$.
  Circles and gray bars: medians and interquartile ranges (IQR) of loss-matched
  empirical looping coefficients across 163 configurations with $T\le8$
  (\Cref{app:law-size-estimates}). Vertical segments join each median to the fitted value at the same loop setting.
  Three coefficients, $c,p,q$, suffice to capture how the runs depend on $T$ and $B$.}
  \label{fig:gain-surface}
\end{minipage}
\end{figure}

\paragraph{Empirical looping coefficients.}
Figure~\ref{fig:backward-nats} motivates coupling
$T$ and $B$: the empirical looping coefficient has a smaller fitted
log--log slope against $r=B/T$ at larger $T$.
Figure~\ref{fig:gain-surface} compares the shared looping-coefficient surface
with empirical looping coefficients. For each looped run with loss $\ell$, we first use
shape-preserving piecewise cubic Hermite interpolation (PCHIP)
\citep{fritsch1984method} to interpolate the non-looped measurements at the
same target token budget to obtain the
effective size $\widehat N_{\mathrm{eff}}$ matching $\ell$.
We then use the effective-size relation in \eqref{eq:law}: subtracting
$N_{\mathrm{once}}$ and dividing by $N_{\mathrm{rec}}$ gives the empirical
looping coefficient $\widehat g$ (\Cref{app:law-size-estimates}). Vertical
segments join each empirical median to the fitted coefficient at the same
loop setting. This agreement shows that $g(T, B)$ with just three parameters 
$c,p,q$ captures the observed loop dependence; Table~\ref{tab:law-fit-summary}
quantifies the law's in-sample fit and leave-one-out (LOO) prediction errors.

\paragraph{Fitting procedure.}
We fit all nine parameters jointly by minimizing the sum of squared
differences between predicted and measured validation losses, giving each
configuration equal weight. The reference scales $N_0=2.4\times10^8$
parameters and $D_0=2^{30}$ tokens are held fixed. Predictions use actual token counts $D$; budget holdouts group target token budgets.
To reduce sensitivity to the initial guess, we fit the law 64 times from
different random starting values and retain the solution with the smallest
fitting error. All fits converge.
Table~\ref{tab:law-design-coefficients} gives the parameters fitted on the
230 configurations used for the validation below. Further fitting details
are in \Cref{app:law}.

\begin{figure}[!htbp]
  \centering
  \includegraphics[width=\textwidth]{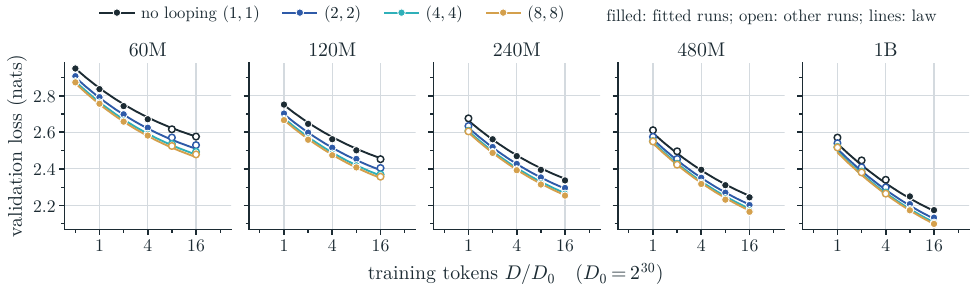}
  \caption{\textbf{More loops with full \ac{BPTT} lower the loss at every size and budget and one law tracks all of
  them.} Validation loss against training tokens for the four loop settings with full \ac{BPTT}. Dots: runs, open where
  outside the fitting set; lines: the full-data fitted law (\Cref{sec:law}).}
  \label{fig:scaling}
\end{figure}

\paragraph{Holdout validation.}
To test whether the law predicts measurements it has not been fitted to,
we leave out selected configurations, fit all nine parameters again
using the remaining measurements, and compare the predictions with the
omitted measurements. For example, to test prediction at an unseen model
size, we exclude every configuration at that size and fit the law using
the other sizes. We repeat this for each size and for the other exclusions
listed in Table~\ref{tab:law-fit-summary}.

\begin{figure}[!htbp]
  \centering  \includegraphics[width=\textwidth]{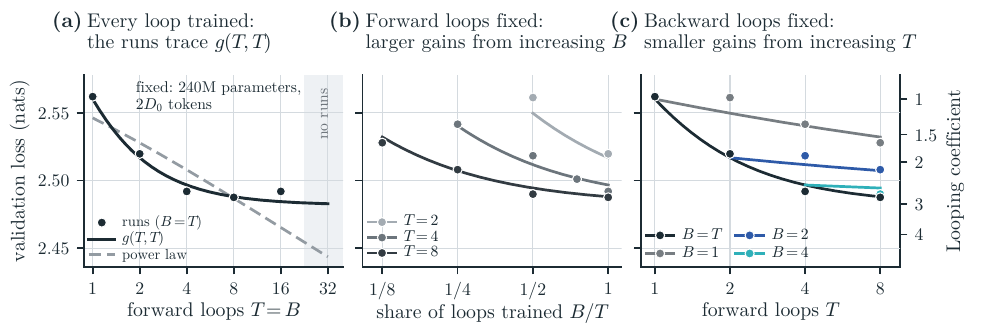}
  \caption{\textbf{Forward and backward loops at fixed size and budget (240M, $2D_0$).} Markers show measured runs and solid lines show the fitted law.
  \textbf{(a)}~$B=T$: the gain levels off; dashed, the power-law gain of prior work \citep{schwethelm2026how} keeps growing.
  \textbf{(b)}~Fixed $T$: training more of the loops lowers the loss.
  \textbf{(c)}~Fixed $B$: more forward loops help little.
  Gains level off, and training more loops ($B$) lowers the loss more than running more loops ($T$).}
  \label{fig:law_tb}
\end{figure}

\paragraph{Prediction accuracy.}
The law tracks the measured loss curves across sizes and token budgets
(Figure~\ref{fig:scaling}), including the held-out configurations (marked with open circles). 
Figure~\ref{fig:law_tb} gives the complementary
view across loop settings at a fixed model size and token budget: varying
$T$ with $B=T$, varying $B$ at fixed $T$, and varying $T$ at fixed $B$.
Both figures use the law fitted on the 230-configuration fitting set; the curves are evaluated on both fitted and held-out configurations.
\Cref{app:grid} overlays the law on the five-size grid of the scaling suite.

Holding out a model size or token budget increases prediction error only
modestly (Table~\ref{tab:law-fit-summary}), whereas unseen loop settings
are harder to predict (\Cref{app:law}). Extrapolation back to the non-looped
baseline is the weakest loop-count check; \Cref{app:law} reports
the boundary errors and alternative gains. The table separates fit quality
from held-out prediction error. \Cref{sec:law-model} examines the
remaining discrepancies in effective-parameter units
(Figure~\ref{fig:law_parts}), and
\Cref{sec:law-fixed-size} checks their effect on compute allocation
(Figure~\ref{fig:alloc_fixed}).

\FloatBarrier

\subsection{Illustrating the scaling law with effective parameters}
\label{sec:law-model}

How many parameters is a loop setting worth? \eqref{eq:law}
allows us to decompose the answer into the gain from fully trained loops, $g(T,T)$, and
the discount for truncation, $g(T,B)/g(T,T)$.
We use the empirical looping coefficients introduced above to compare the law
with the runs in Figure~\ref{fig:law_parts}; their construction is detailed
in \Cref{app:law-size-estimates}.

\begin{figure}[t]
  \centering
  \includegraphics[width=\textwidth]{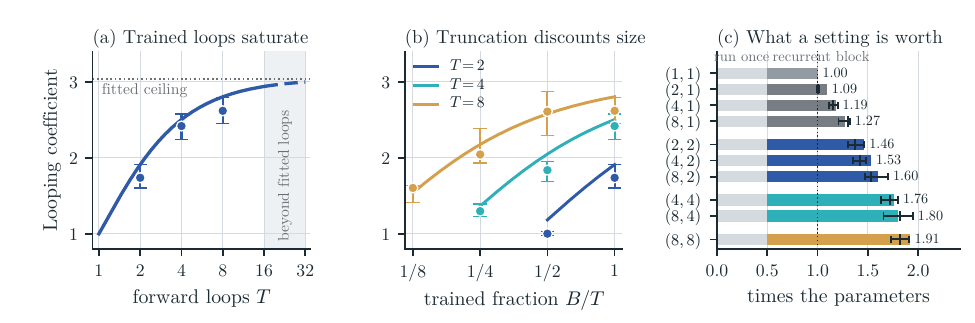}
  \caption{\textbf{What a loop setting is worth in effective size.}
  \textbf{(a)}~Full \ac{BPTT}, the looping coefficient $g(T,T)$, approaches its fitted ceiling (dotted); dashed beyond the fitted loop counts.
  \textbf{(b)}~The looping coefficient $g(T,B)$ against the trained fraction $B/T$.
  \textbf{(c)}~Total effective size $N_{\mathrm{eff}}/N$, grouped by $B$.
  Lines, bars and numbers: the law; markers and error bars: run estimates, median and interquartile range over sizes and budgets
  (\Cref{app:law-size-estimates}).
  Effective size is set mainly by $B$: within each group of equal $B$, neither the law nor the runs gain much from more forward loops.}
  \label{fig:law_parts}
\end{figure}

\paragraph{The looping coefficient saturates for large $T$.}
With full \ac{BPTT}, $B=T$, and the gain is
$g(T,T)=\exp[c\cdot(1-T^{-(p+q)})]$.
The runs in Figure~\ref{fig:law_tb}a vary both loop counts together at
fixed model size and token budget. Each doubling of
trained loops adds less than the one before (Figures~\ref{fig:law_tb}a and
\ref{fig:law_parts}a), as \citet{chen2026scaling} also find with all loops
trained. The power-law gain of prior work \citep{schwethelm2026how} keeps
growing, whereas ours approaches $e^c$. Our runs support the
bounded form over the power law, predicting held-out forward-loop counts
more accurately (\Cref{app:law-bounded}).

\paragraph{Doubling $B$ outweighs doubling $T$ alone.}
At fixed $T$, truncation reduces $B$ and hence $g(T,B)$, consistent with
the smaller effective sizes in Figure~\ref{fig:law_parts}b and higher
losses in Figure~\ref{fig:law_tb}b.
Figure~\ref{fig:law_parts}c clarifies the relative roles of $T$ and $B$
by grouping settings with the same $B$. The fitted bars and
empirical estimates change relatively little as $T$ increases within a
group, but show larger gains as $B$ increases across groups, including
comparisons at the same $T$.
The fitted exponents quantify this asymmetry: $q>p$
(Table~\ref{tab:law-design-coefficients}) implies $g(T,2B)>g(2T,B)$
for $2B\le T$. From the same initial setting at fixed model size,
doubling $B$ thus adds more effective parameters than doubling
$T$ alone; the small loss gains from extra forward loops
at fixed $B$ support the same conclusion (Figure~\ref{fig:law_tb}c). 

\begin{keyfinding}[Loops add effective parameters; truncation discounts them]
Each fully trained loop adds effective parameters, with diminishing
returns toward a finite ceiling. Truncation discounts this gain, and from
the same setting, doubling $B$ adds more effective parameters than
doubling $T$ (Figure~\ref{fig:law_parts}c).
\end{keyfinding}

\section{Truncate or not?}
\label{sec:truncate}

Truncation saves backward compute but reduces the effective parameter
count (\Cref{sec:law-model}). Whether it lowers loss at a given
compute depends on which quantities are held fixed. We first fix the
forward-loop count, as prior scaling studies do, and then let both loop
counts vary at equal compute.

\subsection{Fixed forward loops: truncation can be optimal}
\label{sec:law-fixed-forward}

Suppose the forward-loop count $T$ is prescribed, and the model size $N$,
token count $D$, and backward-loop count $B$ are chosen to minimize loss
at a training compute $C$, with the recurrent block keeping a fixed share
of the parameters as $N$ varies.
Truncation lowers the effective size
$N_{\mathrm{eff}}=N_{\mathrm{once}}+g(T,B)\cdot N_{\mathrm{rec}}$, but it
also lowers the training cost, which we express through the
\emph{compute-equivalent size} $N_{\mathrm{comp}}$:
\begin{equation}
 C\simeq6D\cdot N_{\mathrm{comp}},
 \qquad N_{\mathrm{comp}}=N_{\mathrm{once}}+\frac{T+2B}{3}\cdot N_{\mathrm{rec}}.
 \label{eq:law-leading-compute}
\end{equation}
We count two FLOPs per parameter and token for a forward pass and four for a backward pass,
omitting attention products and the output projection.
We also omit the optimizer update as its cost per step depends on the parameter count but not on $T$ or $B$.
Under this count, the prelude and coda run one forward and one backward pass, costing $6N_{\mathrm{once}}$ FLOPs per token,
and the recurrent block runs $T$ forward and $B$ backward passes, costing $(2T+4B)\cdot N_{\mathrm{rec}}$.
Their sum is $6N_{\mathrm{comp}}$.
A looped model thus corresponds to two different non-looped sizes:
trained on the same tokens, it reaches the loss of a non-looped model with $N_{\mathrm{eff}}$ parameters,
and it costs the compute of a non-looped model with $N_{\mathrm{comp}}$ parameters.
Their ratio $\chi=N_{\mathrm{comp}}/N_{\mathrm{eff}}$ is the compute paid per effective parameter.
Therefore, our optimization at fixed compute reduces to choosing the backward-loop count that minimizes this ratio
(see Lemma~\ref{lem:law-fixed-forward} in \Cref{app:law-allocation} for a rigorous derivation):
\begin{equation}
 B^*\in\operatorname*{arg\,min}_{B\in\{1,\ldots,T\}}\chi(T,B),
 \qquad \text{where}\quad \chi(T,B)=\frac{N_{\mathrm{comp}}}{N_{\mathrm{eff}}}.
 \label{eq:law-cost-per-effective-parameter}
\end{equation}
Since $\chi$ does not involve $C$, neither does $B^*$.
The stationary condition \eqref{eq:law-backward-balance} gives continuous
candidates for $B^*$; integer choices are evaluated directly using $\chi$
(Table~\ref{tab:law-fixed-forward-opt}).

Truncation reduces effective size, but frees compute for a larger model
and more tokens. At $T=8$, the law selects $B^*=2$ and predicts a
34\% compute saving at equal optimized loss
(Table~\ref{tab:law-fixed-forward-opt}). Separately fitted laws at
$T=4,8$ also favor $B=2$ over $B=T$: at $T=8$, their compute-optimal
frontiers predict a 33--35\% saving within the range of the runs
(Figure~\ref{fig:iso-study}; fitting details in \Cref{app:law-allocation}).

\begin{table}[t]
\centering
\captionsetup{width=0.95\linewidth}%
\caption{\textbf{With $T$ prescribed, the law selects $B^*=2$ and so truncates from $T=4$ on.}
For each forward-loop count $T$, $B^*$ is the backward-loop count that minimizes
$\chi(T,B)=N_{\mathrm{comp}}/N_{\mathrm{eff}}$, the compute per effective parameter, over the integers $1\leq B\leq T$,
with $N_{\mathrm{rec}}/N\simeq0.507$ and $N,D$ optimized under leading FLOPs.
$N^*$ and $D^*$ are the optimal model size and token count at $B^*$ relative to full \ac{BPTT} at the same $T$ and compute $C$.
These ratios do not depend on $C$. The gap between $\chi(T,B^*)$ and $\chi(T,T)$, and hence the compute that truncation saves, grows with $T$.}
\label{tab:law-fixed-forward-opt}
\small
\setlength{\tabcolsep}{0pt}
\renewcommand{\arraystretch}{1.25}
\begin{tabular*}{0.85\linewidth}{@{\extracolsep{\fill}}ccrrrr@{}}
\toprule
$T$ & $B^*$ & $\chi(T,B^*)$ & $\chi(T,T)$ & $N^*$ & $D^*$ \\
\midrule
1  & 1 & 1.000 & 1.000 & 1    & 1    \\
2  & 2 & 1.031 & 1.031 & 1    & 1    \\
4  & 2 & 1.204 & 1.430 & 1.27 & 1.08 \\
8  & 2 & 1.580 & 2.376 & 1.51 & 1.20 \\
16 & 2 & 2.343 & 4.339 & 1.70 & 1.31 \\
\bottomrule
\end{tabular*}

\end{table}

\begin{figure}[t]
\centering
\includegraphics{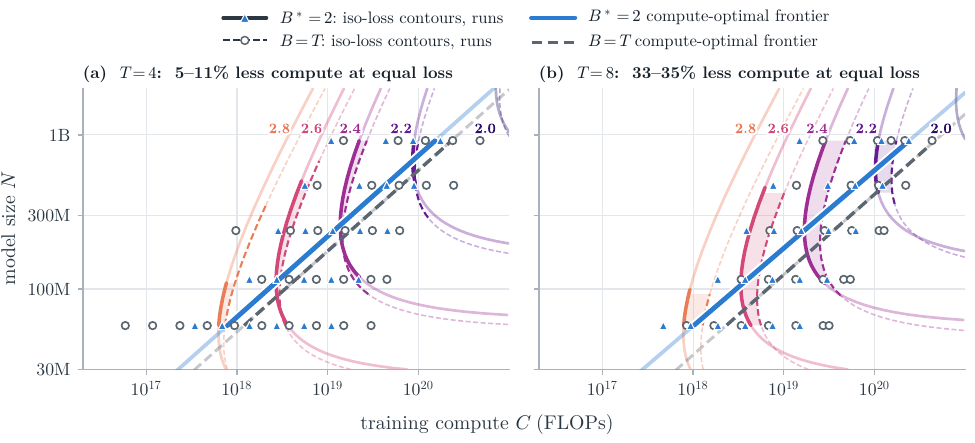}
\caption{\textbf{Separately fitted laws confirm the predicted truncation.} Iso-loss contours (labels: loss in nats) and
compute-optimal frontiers of laws fitted separately to $B=T$ (dashed, grey) and to the predicted $B^*=2$ (solid), at $T=4$ (a)
and $T=8$ (b). Markers are measured runs; faded segments lie outside the runs of a setting. At every loss level, the $B^*=2$
contour lies left of the $B=T$ contour (shaded gap), so truncation reaches the same loss with less compute: 5--11\% less at
$T=4$ and 33--35\% less at $T=8$, where both compute-optimal points lie within the runs. The saving grows with $T$ because further backpropagated loops
add little effective size at full cost. The compute that truncation frees goes partly into a larger model.}
\label{fig:iso-study}
\end{figure}

This agrees with the fixed-compute truncation benefit reported by
\citet{schwethelm2026how}. Parcae uses truncated BPTT but leaves
the compute-optimal choice of forward and backward loop counts to
future work \citep[Appendix~I]{prairie2026parcae}. The benefit is conditional on
prescribing $T$. We next ask whether it survives when both loop counts
can be chosen.

\subsection{Free forward-loop count: full BPTT is optimal}
\label{sec:alloc}

When $T$ can also be chosen, we optimize the split between forward and
backward loops at equal compute. Fix the model architecture and token
count $D$. In our FLOP accounting, each forward pass through the recurrent
block costs one unit and each backward pass costs two (\Cref{app:compute-flops}). The recurrent block thus
costs $T+2B$ units per token; all other model components have the same cost
across loop settings. Holding $T+2B$ fixed therefore holds total compute
fixed: for example, $(T,B)=(4,4)$ and $(8,2)$ both cost 12 recurrent units
per token.

Fix such a loop-compute count $S=T+2B$.
With the architecture and $D$ fixed, the law \eqref{eq:law} depends on the loop setting only through
$N_{\mathrm{eff}}=N_{\mathrm{once}}+g(T,B)\cdot N_{\mathrm{rec}}$, and the loss decreases as $N_{\mathrm{eff}}$ grows.
Since $N_{\mathrm{once}}$ and $N_{\mathrm{rec}}>0$ are fixed, minimizing the loss at this compute
is equivalent to maximizing the looping coefficient $g(T,B)$, which we relax to continuous loop counts:
\begin{equation}
 \underset{T,B\in\mathbb R}{\operatorname{maximize}}\quad g(T,B)\qquad\text{subject to}\quad T+2B=S,\quad 1\leq B\leq T.
 \label{eq:law-loop-split-opt}
\end{equation}

\paragraph{The optimal split uses full BPTT.}
Recall that $g(T,B)=\exp[c\cdot(1-T^{-p}\cdot B^{-q})]$ with $c>0$.
Thus \eqref{eq:law-loop-split-opt} is equivalent to minimizing
$T^{-p}\cdot B^{-q}$ under the same affine constraints. This is a standard
convex optimization problem: whenever $q>2p$, its unique solution is
$T=B=S/3$ for every $S\geq3$, giving full \ac{BPTT}
(Lemma~\ref{lem:law-joint-loops} in \Cref{app:law-joint-loops}). The fitted exponents in
Table~\ref{tab:law-design-coefficients}, $p\simeq0.237$ and
$q\simeq1.022$, satisfy $q>2p\simeq0.474$ by a wide margin. Therefore,
the law predicts that full \ac{BPTT} is optimal at every compute
if both $T$ and $B$ can be chosen continuously.\footnote{The continuous optimum need not be an integer setting.
At the loop compute of $(4,3)$, $T+2B=10$, it is $(10/3,10/3)$, which
lies between the fully trained settings $(2,2)$ and $(4,4)$. Over a
narrow compute range, the law therefore gives $(4,3)$ a lead of at most
0.002 nats over both. Figure~\ref{fig:law-free-size} shows the same effect: $(4,3)$, $(8,4)$ and $(16,8)$ lie just below the
full-BPTT frontier. Their fully trained counterparts at equal loop compute, with $(T+2B)/3$ loops, are not integer
settings, and each needs at most 2.2\% more compute per effective parameter than its counterpart. }

\paragraph{Separately fitted laws confirm the benefit of full BPTT.}
The compute-optimal frontiers of the separately fitted laws predict
that $(4,4)$ reaches the same loss with less compute than $(8,2)$ where
both optimal allocations lie within the ranges of their runs
(Figure~\ref{fig:free-forward-isoloss}).
Matched-compute runs agree: at $D=16D_0$, $(4,4)$ achieves lower loss than $(8,2)$ at every size
(Table~\ref{tab:law-joint-loop-pairs}).

\begin{table}[t]
\centering
\begin{minipage}[t]{0.42\textwidth}
\vspace{0pt}
\captionsetup{type=table,width=\linewidth}
\centering
\caption{\textbf{Reducing forward depth at equal training compute.}
Measured validation losses at $D=16D_0$ (17.2 billion tokens).
Within each row, $(T,B)=(8,2)$ and $(4,4)$ have identical model size,
tokens, and training compute $C$, including optimizer updates.
Losses and the reduction $\Delta L=L_{8,2}-L_{4,4}$ are in nats per token.
Model sizes use nominal labels; $C$ is the training compute in FLOPs.}
\label{tab:law-joint-loop-pairs}
\small
\setlength{\tabcolsep}{0pt}
\renewcommand{\arraystretch}{1.15}
\begin{tabular*}{\linewidth}{@{\extracolsep{\fill}}lrrrr@{}}
\toprule
Size & $C$ & $L_{8,2}$ & $L_{4,4}$ & $\Delta L$ \\[-2pt]
 & {\scriptsize($10^{19}$)} & & & {\scriptsize($10^{-2}$)} \\
\midrule
60M & 2.41 & 2.5211 & 2.4933 & 2.78 \\
120M & 4.59 & 2.3843 & 2.3696 & 1.47 \\
240M & 8.28 & 2.2778 & 2.2642 & 1.36 \\
480M & 15.29 & 2.1845 & 2.1738 & 1.07 \\
1B & 28.38 & 2.1144 & 2.1057 & 0.87 \\
\bottomrule
\end{tabular*}

\end{minipage}\hfill
\begin{minipage}[t]{0.55\textwidth}
\vspace{0pt}
\centering
  \includegraphics[width=\linewidth]{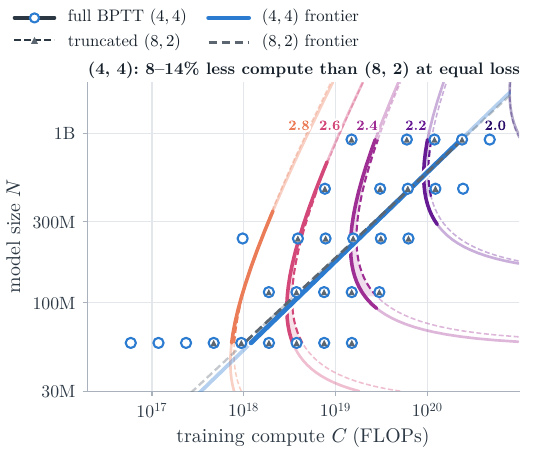}
\end{minipage}
\par\medskip
\captionsetup{type=figure,width=\textwidth}
  \caption{\textbf{Fewer fully trained loops improve the compute-optimal frontier.}
  Iso-loss contours of separately fitted laws for $(T,B)=(4,4)$ (solid)
  and $(8,2)$ (dashed), both with $T+2B=12$. Thick lines: compute-optimal frontiers; faded segments lie outside the runs of a
  setting. At equal loss, solid contours lie left of dashed ones: $(4,4)$ uses 8--14\% less
  compute, so once $T$ can be chosen, full \ac{BPTT} beats truncation.
  }
  \label{fig:free-forward-isoloss}
\end{table}

\begin{keyfinding}[Truncation helps at a prescribed forward-loop count; with both forward and backward loop counts tunable, full BPTT is optimal]
When $T$ is prescribed, the law favors truncation, and separately fitted
laws agree. When both loop counts can be chosen, full \ac{BPTT} is optimal
at every size and compute.
\end{keyfinding}

\section{Allocating compute among loops, data, and model size}
\label{sec:compute}

Given full \ac{BPTT}, it remains to divide training compute among
loops, tokens, and model size. Following
\citet{hoffmann2022training}, we summarize each allocation by how its
optimum scales with compute, $X^*\propto C^{a}$.
We first fix the model size
(\Cref{sec:law-fixed-size}) and then let it vary
(\Cref{sec:law-free-size,sec:untie}).

\subsection{Fixed model size: loops versus data}
\label{sec:law-fixed-size}

Suppose a model must fit within a device's memory budget, so adding
parameters is not an option. With more training compute available,
should we reuse those parameters through more loops per token, or
train on more tokens? At fixed size, more loops raise effective
size but leave fewer training tokens within the same compute budget;
our law balances these two.

\begin{figure}[t]
  \centering
  \includegraphics{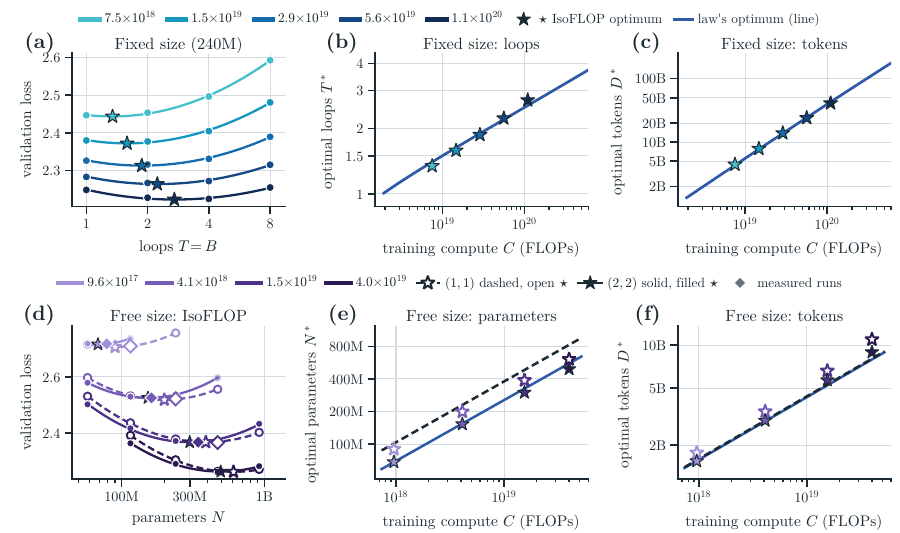}
  \caption{\textbf{How the compute-optimal allocation scales with training compute.} \textbf{Top:} fixed model size (240M).
  (a) Loss against the loop count at five computes, with the optimum at each compute marked by a star. (b, c) The optimal loop
  count $T^*$ and token count $D^*$ of the stars, and the law's optimum (line), whose exponents approach
  \eqref{eq:alloc-fixed-scaling}. \textbf{Bottom:} free model size. (d) Loss against model size at four computes for $(1,1)$
  (dashed) and $(2,2)$ (solid), with their optima (stars) and the compute-optimal pairs of \Cref{sec:untie} (diamonds). (e, f)
  $N^*$ and $D^*$ of the stars and the law's optimum \eqref{eq:alloc-free-scaling}, which places the optimum of $(1,1)$ at a
  larger model than its runs do. Colours mark compute within each row, and compute is $6N_{\mathrm{comp}}D$
  \eqref{eq:law-leading-compute}. At fixed size, added compute goes mostly to tokens while the loop count grows slowly. With
  free size, $(2,2)$ tracks $(1,1)$ at a smaller model.}
  \label{fig:alloc-scaling}
\end{figure}

\paragraph{Asymptotic loop--token allocation under full BPTT.}
The allocation power laws below describe the large-compute limit of
our fitted loss law, not a separate fit to measured optima.
Fix the architecture and parameter counts ($N_{\mathrm{rec}}>0$);
setting $B=T$ in \eqref{eq:law} gives the full-\ac{BPTT} loss
\begin{equation}
 L_N(T,D)=E+\left[
 A\cdot\left(\frac{N_{\mathrm{once}}+N_{\mathrm{rec}}
 \cdot e^{c\cdot(1-T^{-(p+q)})}}{N_0}\right)^{-\alpha}
 +H\cdot\left(\frac{D}{D_0}\right)^{-\beta}
 \right]^{\kappa}.
 \label{eq:law-full-bptt-loss}
\end{equation}
Because $\kappa>0$, minimizing \eqref{eq:law-full-bptt-loss} is equivalent
to minimizing its inner sum, $J_N=(L_N-E)^{1/\kappa}$.
As $T\to\infty$, the effective size approaches
$N_\infty=N_{\mathrm{once}}+e^c\cdot N_{\mathrm{rec}}$. Expanding around this
limit gives
\begin{equation}
 J_N(T,D)=J_\infty+K_N\cdot T^{-(p+q)}
 +H\cdot\left(\frac{D}{D_0}\right)^{-\beta}
 +O\!\left(T^{-2(p+q)}\right),
 \label{eq:law-full-bptt-expansion}
\end{equation}
where $J_\infty=A\cdot(N_\infty/N_0)^{-\alpha}$ and
$K_N=\alpha\cdot c\cdot(N_{\mathrm{rec}}\cdot e^c/N_\infty)\cdot J_\infty>0$ are fixed by
the architecture and fitted coefficients; the leading loop
term has exponent $p+q$ and the data term has exponent $\beta$.

Following the power-law allocation argument of \citet{hoffmann2022training},
we balance these two terms under a compute constraint. At fixed size,
compute per token grows linearly with $T$ at large depth ($C\propto T\cdot D$).
For continuous $T,D$, minimizing the leading terms in
\eqref{eq:law-full-bptt-expansion} gives the asymptotic scaling
\begin{equation}
 T^*\propto C^{\frac{\beta}{\beta+p+q}},\qquad
 D^*\propto C^{\frac{p+q}{\beta+p+q}}
 \quad\text{at fixed }N.
 \label{eq:law-branch-powers}
\end{equation}
\begin{figure}[t]
  \centering
  \includegraphics{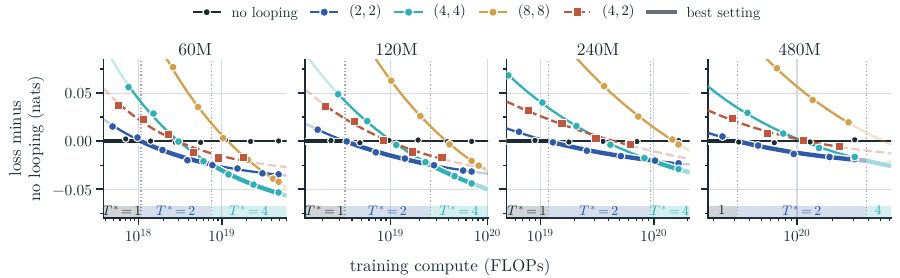}
  \caption{\textbf{At a fixed model size, the best loop count grows with training compute: no looping, then two, then four trained
  loops.} Loss minus that of the model without looping at the same compute, which spends it on more tokens; markers are runs, with a
  smooth curve through each setting; $(4,2)$ trains half of its four loops. The bold curve is the best setting, named in the strip by its
  loop count $T^*$ (faded where extrapolated). Within each stretch the best model spends added compute on tokens, and at each switch it
  trades tokens for loops.}
  \label{fig:alloc_fixed}
\end{figure}

These powers need not describe finite-compute optima, where overheads
and integer loop counts remain consequential.
\Cref{app:law-fixed-size} derives this limit from the exact stationary
condition, including the compute overhead of the non-recurrent layers.

\paragraph{Loops grow slowly; most added compute goes to data.}
Substituting the fitted $p+q\simeq1.259$ and $\beta\simeq0.393$ into
\eqref{eq:law-branch-powers} gives, with rounded exponents,
\begin{equation}
 T^*\propto C^{0.24},\qquad D^*\propto C^{0.76}.
 \label{eq:alloc-fixed-scaling}
\end{equation}
A tenfold increase in compute therefore raises the optimal loop count by
about 1.7 times and the token count by about 5.8 times. These exponents
are asymptotic: because loop counts are integers, the best setting
changes in steps. Between steps, added compute goes to
tokens; at each step, the model switches to a deeper setting with fewer
tokens. The runs show the same trend: at each size, the loop count
minimizing loss at equal compute grows with compute as predicted
(Figure~\ref{fig:alloc-scaling}a--c).

\paragraph{Comparison with prior scaling laws.}
Parcae reports empirical exponents of approximately 0.40 for mean
recurrence and 0.78 for training tokens \citep[Section~5.2]{prairie2026parcae}.
\eqref{eq:alloc-fixed-scaling} agrees that tokens grow faster than loops,
but predicts slower loop growth in the asymptotic regime. Parcae samples
recurrence and uses truncated BPTT, so its reported optimum is conditional
on that training recipe. Our full-\ac{BPTT} result suggests that
truncation may be suboptimal when both loop counts are free under our law
and compute model; whether this also holds for Parcae requires a direct
comparison. The law of \citet{schwethelm2026how} instead lets the looping
coefficient grow unbounded as $T^{\varphi}$. Our gain saturates, and $p+q$
governs the approach to that ceiling, giving a different asymptotic
allocation between loops and tokens.

\begin{keyfinding}[At fixed model size, optimal loops and tokens both grow with compute; tokens grow faster]
Under the fitted law with full \ac{BPTT}, the compute-optimal loop count
grows asymptotically as $C^{0.24}$ and the token count as
$C^{0.76}$. The runs at every size agree.
\end{keyfinding}

\subsection{Free model size: the law favors no looping}
\label{sec:law-free-size}
\begin{figure}[t]
\centering
\begin{minipage}[t]{0.48\textwidth}
\vspace{0pt}\centering
\includegraphics{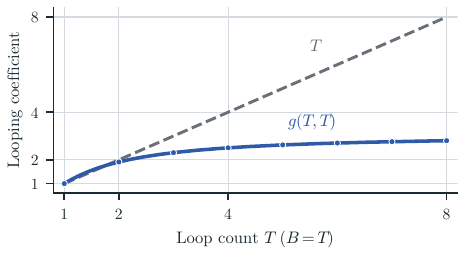}
\captionsetup{width=\linewidth}
\caption{\textbf{Loops versus untied copies.} Fitted $g(T,T)$ (solid, with markers at integer $T$) against
 $T$ untied copies (dashed). It stays below $T$ for $T\ge2$, so untied copies are worth more. The gap is smallest at $T=2$, where two loops come close to the effective size of two untied copies.}
\label{fig:loop-gain-vs-weights}
\end{minipage}\hfill
\begin{minipage}[t]{0.48\textwidth}
\vspace{0pt}\centering
\includegraphics{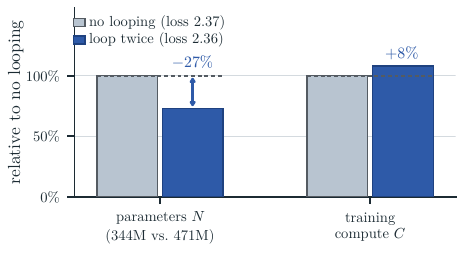}
\captionsetup{width=\linewidth}
\caption{\textbf{Two fully trained loops against no looping at $2.1\times10^{19}$ FLOPs}, as in Figure~\ref{fig:teaser}d.
 Parameters and compute of the $(2,2)$ model relative to the compute-optimal model without looping.
 With 8\% more compute, $(2,2)$ reaches a lower loss with 27\% fewer non-embedding parameters.}
\label{fig:pair-bars}
\end{minipage}
\end{figure}
When model size is free, the compute spent repeating a block can instead
train more unique parameters. Start with a $T$-loop model using full
\ac{BPTT}, for integer $T\geq2$, and replace its shared recurrent block
with $T$ untied copies. The resulting non-looped model has the same
compute $C$ and token budget $D$ under \eqref{eq:law-leading-compute},
but greater effective size:
\begin{equation}
 N_{\mathrm{once}}+g(T,T)\cdot N_{\mathrm{rec}}
 <N_{\mathrm{once}}+T\cdot N_{\mathrm{rec}}.
 \label{eq:law-nonloop-dominance}
\end{equation}

The left side is the effective size of the looped model; the right side
is that of the model with $T$ untied copies.
The inequality follows from the fitted $g(T,T)<T$
(Figure~\ref{fig:loop-gain-vs-weights}; \Cref{app:law-free-size}). The non-looped model therefore has greater
effective size and lower predicted loss at every token budget.
Figure~\ref{fig:loop-gain-vs-weights} also shows that $g(2,2)$ is close
to $2$; we examine this two-loop case in \Cref{sec:untie}.

\paragraph{The best loop setting does not change with compute.}
With parameter proportions fixed and positive, continuous $N,D$
unrestricted, every loop setting shares the compute exponents of the
model without looping. Optimizing the loop setting as well gives
\begin{equation}
 N^*\propto C^{0.56},\qquad D^*\propto C^{0.44},\qquad
 T^*=B^*=1.
 \label{eq:alloc-free-scaling}
\end{equation}
The exponents are rounded predictions under the leading-FLOP count
(\Cref{app:law-free-size}). Under these assumptions, the ranking of
loop settings is the same at every compute.

\paragraph{No looping is optimal, and two fully trained loops come closest.}
Among the fifteen settings, the law favors the model without looping
at fixed compute (Figure~\ref{fig:law-free-size}), consistent with \citet{schwethelm2026how}.
The best looped setting is $(2,2)$
(Table~\ref{tab:law-fixed-forward-opt}): $g(2,2)\approx2$, so two
shared loops provide nearly the effective size of two untied copies.
Because
$N_{\mathrm{eff}}>N$ for $(2,2)$, the law predicts that it comes close
to the model without looping with fewer trainable parameters.
\Cref{sec:untie} tests this.

\begin{keyfinding}[With free model size, no looping is optimal]
At every compute, with model size and token count free, the law
favors the model without looping
(Figure~\ref{fig:law-free-size}).
\end{keyfinding}

\subsection{Two fully trained loops are competitive with fewer parameters}
\label{sec:untie}

\paragraph{Near parity at equal compute.}
At each of three training budgets, we train a model with two fully
trained loops, $(2,2)$, and a model without looping, each at the size and
token count that minimize its loss for that budget, with three seeds per
model (\Cref{app:untie}). At every budget, the looped model stays within
0.008 nats of the model without looping while having 27--32\% fewer
non-embedding parameters (Table~\ref{tab:compute-optimal-pairs}). The gap
shrinks as the budget grows, and at $2.1\times10^{19}$
FLOPs\footnote{Including the optimizer update, as in
Table~\ref{tab:compute-optimal-pairs}; the forward and backward passes alone
take $2.0\times10^{19}$ FLOPs.} it is smaller than its standard error across
seeds. Both models sit near the flat minima of their IsoFLOP profiles, the
looped one at a smaller size (Figure~\ref{fig:alloc-scaling}d).

\paragraph{Ahead with a little more compute.}
Because the gap is so small, a little more compute lets the looped model
pull ahead. We take the pair at $2.1\times10^{19}$ FLOPs and train its
$(2,2)$ model on 8\% more tokens, and so with 8\% more compute. Its loss
then falls 0.007 nats below that of the model without looping, well beyond
the spread across seeds, while it still has 27\% fewer non-embedding
parameters (Figure~\ref{fig:pair-bars}; Table~\ref{tab:compute-optimal-pairs}).
Appendix~\ref{app:untie} also compares the looped model with an untied model, which
gives each loop its own block, at a fixed number of training tokens.

\begin{keyfinding}[Two fully trained loops nearly match the model without looping with 27--32\% fewer parameters]
In the runs, the compute-optimal $(2,2)$ model stays within 0.008 nats of
the compute-optimal model without looping at equal training compute, with
27--32\% fewer non-embedding parameters; with 8\% more compute, its loss
falls below.
\end{keyfinding}

\FloatBarrier
\section{Conclusion}
\label{sec:conclusion}

Truncated \ac{BPTT} saves compute, but whether it lowers the loss depends on which quantities are held fixed. We fit a scaling
law that expresses every loop setting as effective parameters, with a gain that saturates and shrinks under truncation
(\Cref{sec:law}). Under this law, truncation helps only when the forward-loop count is prescribed. When both loop counts are chosen,
full \ac{BPTT} is optimal at every model size and compute we analyze (\Cref{sec:truncate}).

The law also determines how to allocate compute. At fixed model size, most additional compute should go to data, and the number
of fully trained loops grows only as about $C^{0.24}$ (\Cref{sec:law-fixed-size}). When model size is free, the law favors no
looping, but two fully trained loops come close: with 27--32\% fewer non-embedding parameters they nearly match the compute-optimal
non-looped model at equal compute, and with 27\% fewer parameters they achieve lower loss at 8\% more
compute (\Cref{sec:untie}).

\paragraph{Limitations.}
Our main limitation is the focus on dense models: we do not study mixture-of-experts (MoE) architectures, so whether our
conclusions on truncation and compute allocation extend to looped MoE models remains open. Second, we measure validation loss
only and do not establish whether the observed gains translate to downstream task performance or reasoning ability.
Third, our experiments are capped at $T\le16$, so the gain ceiling and later switches to more loops remain extrapolations.
Related scaling sweeps reach eight loops in \citet{schwethelm2026how} and mean training loop counts of twelve in
Parcae \citep{prairie2026parcae}. Much deeper recurrence can exacerbate training instability, including residual explosion
and loss spikes \citep{prairie2026parcae,fu2026simply}; developing stable training recipes for that regime is beyond our scope.
Finally, we fix the ratio of non-embedding parameters in the combined prelude and coda to those in the recurrent block
for each model size. Varying this ratio to study how parameter allocation affects the gains from looping and the
compute-optimal loop counts is a direction for future work.

\section*{Acknowledgments}
We thank Yuwei Niu, Yi Ge, and Tianyu Fu for valuable feedback and helpful discussions.

\newpage
\bibliographystyle{ims}
\bibliography{refs}

\newpage
\appendix

{\makeatletter\patchcmd{\l@section}{\addvspace{1.0em \@plus\p@}}{\addvspace{0.3em \@plus\p@}}{}{}\makeatother
 \color{sectionblue}\tableofcontents}
\clearpage
\section{Extended related work}
\label{app:related}

\paragraph{Training through recurrent depth.}
Truncated \ac{BPTT} limits gradient propagation to a subset of recurrent steps, saving activation memory and backward compute at the cost of biased gradients \citep{williams1990efficient,tallec2017unbiasing}.
Looped language models often supervise the final output and backpropagate through only the last few loops.
Huginn and retrofitted models use at most eight backward steps \citep{geiping2025scaling,mcleish2025teaching}, Parcae and Iso-Depth backpropagate through half of the loops \citep{prairie2026parcae,schwethelm2026how}, URM runs the first loops without gradients \citep{gao2025urm}, and looped MoE models supervise within recurrent segments \citep{he2026looping}.
The same truncation trains looped models converted from pretrained language models \citep{park2026loopus}, recurrent-depth models for robot control and dynamics forecasting \citep{tur2026rdvla,lai2026ttf}, and looped refinement in diffusion language models \citep{sun2026freda}.
Recursive reasoning models use even shorter backward horizons.
Hierarchical recurrent models backpropagate through only the final steps \citep{wang2025hierarchical,wang2026hrmtext,han2026hierarchical}, and TRM and its successors run all but the last recursion without gradients \citep{jolicoeur2025trm,sghaier2026ptrm,asadulaev2025dis,baek2026gram,hakimi2026rsm,freinschlag2026serrm,manzari2026recursivevlm,salvatori2026trmp}.
For algorithmic extrapolation, recurrent networks and looped transformers run a random number of iterations without gradients before the supervised ones, or compute the loss over a window of the last loops \citep{bansal2022endtoend,labovich2026stability,yang2024looped}.
Implicit models replace backpropagation through the solver with implicit differentiation \citep{bai2019deep} or with one-step, Jacobian-free, and phantom gradients that differentiate only a few steps at the fixed point \citep{geng2021attention,fung2022jfb,geng2021training}.
Such approximations train implicit models for optical flow, diffusion, and language modeling \citep{bai2022deqflow,pokle2022deqdiffusion,bai2024fpdm,schone2025implicit,feinashley2026attractor}, and truncated \ac{BPTT} in looped transformers can be read as a truncated Neumann series of the fixed-point gradient \citep{movahedi2026fixedpoint,huang2026towards}.
Few studies measure what truncation costs.
\citet{feinashley2026attractor} find that the one-step gradient nearly matches the perplexity of the exact implicit gradient at a fraction of its memory, \citet{huang2026eqr} find that truncation lowers or raises Sudoku accuracy depending on the latent structure, \citet{salvatori2026trmp} find that the required gradient window grows with the propagation distance of the graph, and \citet{ren2026hrm} argue that the one-step gradient delays the stabilization of hierarchical reasoning models.
These studies compare training recipes at a fixed forward depth; we vary the backward depth independently of the forward depth, weighing the compute that truncation saves against the loop gain it costs.

\paragraph{Scaling laws for recurrence.}
Size--data scaling laws quantify how parameters and training tokens jointly determine loss \citep{kaplan2020scaling,hoffmann2022training,videau2026skaling}, and their compute-optimal predictions depend on how the laws are fitted and compute is counted \citep{besiroglu2024chinchilla,porian2024resolving,pearce2024reconciling,choshen2024hitchhiker,li2025farseer}.
Extensions account for architectural choices such as MoE sparsity \citep{clark2022unified,krajewski2024scaling,abnar2025parameters}.
For looped models, some laws fit each architecture separately \citep{wang2026smelt,lee2026sparse,chen2026growth}, while others model the loop count explicitly (Table~\ref{tab:looped_studies}) through additive loss terms \citep{zhu2025scaling} or effective parameters with linear, power-law, or bounded gains \citep{prairie2026parcae,schwethelm2026how,chen2026scaling}.
Among these scaling-law studies of looped transformers, Iso-Depth \citep{schwethelm2026how} and Parcae \citep{prairie2026parcae} are particularly related to our work.
Iso-Depth fits separate laws for full and half-length backpropagation and finds that truncation weakens the loop gain but lowers the loss at a fixed forward-loop count by freeing compute.
At fixed model size, Parcae finds that loops and tokens should grow together with compute, but ties the backward loops to half the forward loops and leaves their joint choice open.
Our law instead treats the backward-loop count as a separate variable, combining coupled size--data scaling \citep{videau2026skaling} with  bounded loop gains \citep{chen2026scaling}.
It recovers the truncation benefit at a prescribed forward-loop count but favors full \ac{BPTT} when both counts are free (\Cref{sec:truncate}).
In looped MoE models, two loops are the common choice \citep{he2026looping}: SMELT finds them better than three or four under matched budgets \citep{wang2026smelt}, and Loopie finds the marginal benefit of recurrence largest at two \citep{gao2026loop}.
For dense models, our law ranks two fully trained loops first among looped settings, since they provide nearly the effective size of two untied copies of the block while further loops add less (\Cref{sec:law-free-size}).
In our runs, two-loop models nearly match compute-optimal models without looping at equal compute while having substantially fewer non-embedding parameters (\Cref{sec:untie}).

\section{Model architecture and experimental setup}
\label{app:arch}

\subsection{Model architecture}

Figure~\ref{fig:arch} draws the model of \Cref{sec:prelim} for the loop setting $(4,2)$. In a training step
(Figure~\ref{fig:arch}a), the prelude maps the tokens to $s_0$, every loop reads $s_0$ again through the input injection, and
the state $s_{T-B}$ entering the last $B$ loops is treated as a constant. The gradient therefore reaches the coda, the last
$B$ loops and, through the injection, the prelude; the first $T-B$ loops run forward only, and evaluation runs all $T$ loops.
Each loop (Figure~\ref{fig:arch}b) applies $W_{\mathrm{inj}}$ to $[\,s_0\,;\,s_{t-1}\,]$, the $\ell_r$ pre-norm transformer
layers of $\mathcal{R}$ and an RMSNorm, with the same weights in every loop. At 2 FLOPs per parameter and token for a forward pass and 4
for a backward pass (Figure~\ref{fig:arch}c), the prelude and coda cost 6, each forward loop 2 and each backward loop 4 more,
giving $6N_{\mathrm{once}}+(2T+4B)N_{\mathrm{rec}}=6N_{\mathrm{comp}}$ FLOPs per token \eqref{eq:law-leading-compute}, before attention and
output \eqref{eq:flops}; the optimizer adds a cost per step \eqref{eq:compute}.

\subsection{Full and truncated backpropagation through the loops}
\label{app:bptt}

This appendix writes out the gradients that full and truncated \ac{BPTT} compute for the recurrence of \eqref{eq:loop}, and the
fixed-point argument that motivates truncation (\Cref{sec:prelim}).

\paragraph{The truncated gradient.}
A loop setting $(T,B)$ executes $T$ forward loops and backpropagates through only the last $B$: the state $s_{T-B}$ is detached from the computation graph, so the first $T-B$ loops run without gradients.
To make this precise, we denote for loop $t$ of \eqref{eq:loop} the Jacobian with respect to the incoming state and the derivative with respect to $\theta$ with the incoming state fixed by
\begin{equation}
  \mathcal{J}_t \;=\; \frac{\partial\Phi_\theta(s_{t-1};s_0)}{\partial s_{t-1}},
  \qquad
  \partial_\theta\Phi_t \;=\; \frac{\partial\Phi_\theta(s_{t-1};s_0)}{\partial\theta}.
\end{equation}
The loss $\mathcal{L}$ depends on $\theta$ only through $s_T$, so $\nabla_\theta\mathcal{L}=(\partial\mathcal{L}/\partial s_T) \cdot (\mathrm{d}s_T/\mathrm{d}\theta)$, and it suffices to compare $\mathrm{d}s_T/\mathrm{d}\theta$ under full and truncated \ac{BPTT}.
The chain rule sums one contribution per loop, and detaching $s_{T-B}$ keeps only the last $B$, giving the truncated gradient $g_B$:
\begin{equation}
  \frac{\mathrm{d}s_T}{\mathrm{d}\theta}
  \;=\; \sum_{t=1}^{T}\mathcal{J}_T\cdots\mathcal{J}_{t+1}\,\partial_\theta\Phi_t,
  \qquad
  g_B \;=\; \sum_{t=T-B+1}^{T}\mathcal{J}_T\cdots\mathcal{J}_{t+1}\,\partial_\theta\Phi_t ,
  \label{eq:bptt}
\end{equation}
where the empty product for $t=T$ is the identity.
Truncation keeps the $B$ contributions with the fewest Jacobian products and drops the rest. The prelude gradient is truncated in the same way.

\paragraph{Truncation from a fixed point view.}
Whether the dropped contributions matter depends on the loop dynamics, and the deep-equilibrium view \citep{bai2019deep,wang2025hierarchical} gives an idealized answer.
Suppose the loops approach a fixed point $s^\star=\Phi_\theta(s^\star;s_0)$, so that $\mathcal{J}_t\approx\mathcal{J}^\star$ and $\partial_\theta\Phi_t\approx\partial_\theta\Phi^\star$ for the loops that carry gradient.
If the spectral radius of $\mathcal{J}^\star$ is below one, the implicit function theorem gives the gradient of the fixed point, and a Neumann series expands it into the same per-loop contributions:
\begin{equation}
  \frac{\partial s^\star}{\partial\theta}
  \;=\; \bigl(I-\mathcal{J}^\star\bigr)^{-1}\partial_\theta\Phi^\star
  \;=\; \sum_{k=0}^{\infty}\bigl(\mathcal{J}^\star\bigr)^{k}\,\partial_\theta\Phi^\star,
  \qquad
  g_B \;\approx\; \sum_{k=0}^{B-1}\bigl(\mathcal{J}^\star\bigr)^{k}\,\partial_\theta\Phi^\star .
  \label{eq:neumann}
\end{equation}
Truncating to $B$ loops thus approximates the fixed-point gradient by the first $B$ terms of its Neumann series, whereas full \ac{BPTT} uses the first $T$.
The extreme case $B=1$ keeps only the leading term $\partial_\theta\Phi^\star$; this is the one-step gradient used to train the Hierarchical Reasoning Model \citep{wang2025hierarchical}.
When $\mathcal{J}^\star$ is contractive, the dropped terms shrink geometrically, which makes truncation a plausible approximation at a fraction of the backward compute.
This idealized argument, however, does not settle whether truncation pays off in practice, for two reasons.
First, it assumes that the loops reach a fixed point, whereas looped language models are usually trained with only a few loops \citep{he2026looping,wang2026smelt,gao2026loop}.
With so few iterations, the states need not come close to a fixed point, so the approximation $\mathcal{J}_t\approx\mathcal{J}^\star$ behind \eqref{eq:neumann} can fail and the contributions dropped in \eqref{eq:bptt} need not be small.
Second, even at a fixed point, the dropped terms are nonzero for $B<T$, so $g_B$ is a biased estimate of the full gradient.
The argument only bounds the error of a single gradient, whereas this bias enters every update, and how it accumulates over training and shapes the final model is not known.
Prior measurements compare recipes at a fixed forward-loop count (\Cref{sec:related}) and so cannot separate the compute that truncation saves from the loop gain it costs.
We therefore measure the net effect directly: the scaling law of \Cref{sec:law} asks whether the saved compute outweighs this bias in the loss that training actually reaches.

\subsection{Experimental setup}
\label{app:setup}

\paragraph{Models.}
We take the architecture and the training recipe of \citet{schwethelm2026how}, the reference recipe below, which builds on the
prelude--recurrent--coda layout of Huginn \citep{geiping2025scaling} (\Cref{sec:prelim}). Its blocks use pre-norm
attention with query--key normalization and rotary position embeddings, head dimension 128, a squared-ReLU feed-forward layer of
width $4d$, no biases, and untied input and output embeddings. We change only what is held fixed. \citet{schwethelm2026how} hold the
executed depth fixed, so that more loops mean fewer unique layers; we hold the unique parameters fixed, as
\citet{prairie2026parcae} do, so that a loop setting changes the compute and not the model. The main grid spans five model sizes, from
58M to 919M non-embedding parameters (Table~\ref{tab:models}); width and depth grow together, and the recurrent block with the
input injection holds 51\% of the non-embedding parameters at four sizes and 56\% at 120M. Short size names round $N$; including the untied input and output embeddings, total parameter counts run from about 107M to about 1.06B. We quote budgets per parameter as $D/N$.
The scaling suite supplements this grid with two compute-matched pairs, adding the 79M and 162M models
at one configuration each (\Cref{app:untie}).

\paragraph{Loop settings and budgets.}
The scaling suite spans 1.1--1{,}291 training tokens per parameter and uses fifteen loop settings.
Ten core settings, $(1,1)$; $(2,1)$ and $(2,2)$; $(4,1)$, $(4,2)$, $(4,3)$ and
$(4,4)$; and $(8,2)$, $(8,4)$ and $(8,8)$, cover $T\in\{1,2,4,8\}$ forward loops, full and truncated \ac{BPTT}, and
ratios $B/T$ from $\tfrac14$ to 1. Five more settings extend them: $(8,1)$, and $T=16$ with $B=1$, 4, 8 and 16 at three
sizes and short budgets (\Cref{app:law-t16,app:grid}). The part of the training compute that depends on the loops grows with
$T+2B$, since a backward loop costs twice a forward loop (\Cref{app:compute-flops}); $(4,1)$ thus costs as much per
token as $(2,2)$, and $(8,2)$ as much as $(4,4)$. Token budgets $D$ are given in units of $D_0=2^{30}\approx1.07\times10^{9}$
tokens. We group runs by their target token budget; actual token counts may differ
slightly because of batch rounding. Predictions use the actual token count $D$.
An experimental configuration is a choice of $(N,D,T,B)$; a run is one
training execution of that configuration. Each configuration contributes
one accepted validation loss to the analysis.

\paragraph{Data and evaluation.}
We train on the 100B-token sample of FineWeb-Edu \citep{penedo2024fineweb}, the corpus of the looped scaling study of
\citet{lee2026sparse}, in a pre-shuffled
release,\footnote{\url{https://huggingface.co/datasets/karpathy/fineweb-edu-100b-shuffle}} tokenized with the Llama~2
tokenizer \citep{touvron2023llama} of the reference recipe \citep{schwethelm2026how} into sequences of 2{,}048 tokens. We report cross-entropy in nats per token on the last
two shards of the release, which no run trains on, after removing the documents whose exact text also appears in the training
data (the same web page often occurs in several crawls).

\paragraph{Optimization.}
We use the optimizers of the reference recipe \citep{schwethelm2026how}. MuonH, a variant of Muon that keeps each weight matrix at a fixed
Frobenius norm \citep{jordan2024muon,wen2025hyperball}, trains the hidden weight matrices (attention, feed-forward and input injection), and
AdamW \citep{loshchilov2019decoupled} trains the embedding, output and normalization parameters, both without weight decay. The MuonH update adds 0.06--5.7\% to the compute of
the forward and backward passes; the analysis leaves it out, and the measured compute we report includes it
(\Cref{app:compute-flops}). The peak learning rates follow the reference's
transfer rule \citep{schwethelm2026how,yang2021tensor,ren2026rethinking}, a power law in the batch and training tokens with the token
exponent re-estimated for our budgets, and the global batch grows as $64\sqrt{D/D_0}$ sequences, following the square-root
scaling of the loss-optimal batch in prior work \citep{li2025predictable,bergsma2025power,filatov2025optimal}
(\Cref{app:lrbs}).
The learning rate warms up over the first 2\% of steps and decays with a cosine to 10\% of its peak; gradients
are not clipped.

\paragraph{Runs and compute.}
The results in this paper come from 592 training runs with $2.7\times10^{22}$ FLOPs of training compute, counted as in
\Cref{app:compute-flops} with the optimizer update. The scaling suite contains 397 runs behind 347 configurations:
385 runs in the five-size grid (\Cref{app:grid}) and 12 runs in two compute-matched pairs (\Cref{app:untie}).
Each configuration's loss is the mean of its accepted runs. The remaining 195 distinct runs support the other experiments.
These experiments comprise four more runs at 480M in
Figure~\ref{fig:alloc_fixed}, nine further runs of the compute-optimal pairs and the 36 of the untied comparison
(\Cref{app:untie}), and the 66 learning-rate and 124 batch-size runs of Figures~\ref{fig:app-lr} and~\ref{fig:app-bs}; 44
runs belong to two of these groups and are counted once. Counting cancelled runs, runs replaced by reruns, pilot runs and
earlier calibration sweeps that the paper does not report, the study trained more than 1{,}100 runs with about
$4.8\times10^{22}$ FLOPs.

\section{The law in detail}
\label{app:law}

This appendix details the measured patterns behind the law, its fit and validation, the bounded gain,
and the runs at sixteen loops.

\subsection{What the runs show before any law is fitted}
\label{app:law-runs}
The measured losses separate the effects of forward and backward loops without a fitted law (\Cref{app:grid}). Forward loops that
receive no gradient add almost nothing at first: at fixed $B$, doubling $T$ from $B$ to $2B$ changes the loss by less than 0.007 nats
in all 24 size--budget combinations trained in all ten core settings, while doubling it again, from $2B$ to $4B$, lowers the
loss in all 24. Spending the same compute per token on backward loops is better: $(2,2)$ beats $(4,1)$ and $(4,4)$ beats $(8,2)$ in all 24,
by 0.023 nats on average for $(2,2)$ against $(4,1)$. Backward loops in turn have diminishing returns: with full \ac{BPTT},
every doubling of $T=B$, from 1 to 8, lowers the loss, by less each time, in all 24 combinations where all four settings were
trained.

\subsection{Empirical looping coefficient}
\label{app:law-size-estimates}
The looping coefficient $g(T,B)$ in \eqref{eq:law} rescales the recurrent
block's parameter count. To estimate it from the runs, we first find the
non-looped model size that matches a looped run's measured loss at the
same token budget. This defines the empirical effective size, which we
then convert to a looping coefficient using the run's parameter counts.

\paragraph{Reference interpolation.}
Let $d=D_{\mathrm{target}}/D_0$ denote a target budget level, where $D_{\mathrm{target}}$ is the target token count and
$D_0=2^{30}$ tokens. At each $d$, we collect the accepted non-looped
$(T,B)=(1,1)$ configurations and order their actual parameter counts and measured
losses as $(n_j,\ell_j)_{j=1}^{m_d}$, where
$n_1<\cdots<n_{m_d}$. Each configuration contributes its accepted loss once,
following the same aggregation rules as the fit.
We require at least three distinct sizes and strictly decreasing losses,
$\ell_1>\cdots>\ell_{m_d}$, as holds for the eligible reference groups.
We interpolate log size as a function of loss:
\begin{equation}
 \phi_d(\ell) \coloneqq
 \operatorname{PCHIP}\!\left[
   (\ell_j,\log n_j)_{j=m_d,\ldots,1}
 \right](\ell),
 \qquad \ell\in[\ell_{m_d},\ell_1].
 \label{eq:empirical-size-interpolation}
\end{equation}
Here PCHIP denotes shape-preserving piecewise cubic Hermite interpolation
\citep{fritsch1984method}, with reference points ordered by increasing loss.
The interpolant in \eqref{eq:empirical-size-interpolation} passes through
each measured point, $\phi_d(\ell_j)=\log n_j$, and preserves monotonicity
between points. This loss-to-size conversion uses only the non-looped
measurements, independent of the fitted coefficients of \eqref{eq:law}.

\paragraph{Coefficient of a looped configuration.}
For a looped configuration $i$, let $L_i$ be its measured loss, $d_i$ its target
budget, and $N_{\mathrm{once},i},N_{\mathrm{rec},i}>0$ its
once-run and recurrent counts; when its reference group
is eligible and $L_i\in[\ell_{m_{d_i}},\ell_1]$, define
\begin{equation}
 \widehat N_{\mathrm{eff},i}\coloneqq\exp\!\bigl[\phi_{d_i}(L_i)\bigr],
 \qquad
 \widehat g_i\coloneqq
 \frac{\widehat N_{\mathrm{eff},i}-N_{\mathrm{once},i}}
      {N_{\mathrm{rec},i}}.
 \label{eq:empirical-looping-coefficient}
\end{equation}
In \eqref{eq:empirical-looping-coefficient}, we subtract the once-run
parameters and divide by the recurrent block's size, as in
\eqref{eq:law}. Each empirical coefficient depends on the
run's size and token budget; $g(T,B)$ is the law's shared prediction.
We report an estimate only when the reference group is eligible and the
measured loss lies within its observed range, without extrapolation.
Budget matching uses the target level $d$, keeping batch-rounded budgets
in the same group. Actual consumed tokens differ from the target token budget
by less than $0.2\%$ in these runs; we leave these differences uncorrected.

For example, the 239.24M non-embedding-parameter run at $T=B=2$ and $D=4D_0$ has loss
$2.4272$ nats per token. The reference interpolation gives
$\widehat N_{\mathrm{eff}}=345.20$M, hence
$\widehat g=(345.20-117.98)/121.26\simeq1.874$, using its own once-run
and recurrent counts in millions of parameters.

\paragraph{Summaries across configurations.}
For each setting $(T,B)$, let $\mathcal I_{T,B}$ be the set of eligible
size--budget configurations. Figure~\ref{fig:gain-surface} shows the median of
$\{\widehat g_i:i\in\mathcal I_{T,B}\}$, with the 25th and 75th
percentiles as the interquartile range. Configurations have equal weight; for
sorted estimates $x_{(1)},\ldots,x_{(n)}$, a percentile at probability
$u$ linearly interpolates adjacent order statistics at rank
$1+(n-1)u$. Figure~\ref{fig:law_parts} uses the same loss-matched estimates,
with total-size ratios shown where indicated. The surface figure includes
163 eligible configurations with $T\le8$: 121 from the fitting set and 42 from
the separately held-out set. Reference configurations may also come from either
set, and estimates that use the same references are dependent. These
figures describe agreement with the runs rather than independent
validation of the law. The interquartile ranges show variation across
sizes and budgets, not confidence intervals. \Cref{app:law-fit} reports
the separate holdout evaluations in loss units.

Figure~\ref{fig:backward-nats} uses the same estimator on the
150 eligible configurations with $T\in\{2,4,8\}$ and $1/4\le r=B/T\le1$.
Markers show medians and bars show minima and maxima. For each $T$,
we fit an intercept $a_T$ and slope $s_T$ by ordinary least squares with equal weight per cell:
\begin{equation}
 (\widehat a_T,\widehat s_T)
 =\operatorname*{arg\,min}_{a,s}
 \sum_{i:T_i=T}\left[a+s\log_2r_i-\log_2\widehat g_i\right]^2.
 \label{eq:empirical-looping-slope}
\end{equation}
In \eqref{eq:empirical-looping-slope}, the fitted slopes are $0.823$, $0.446$
and $0.173$ for $T=2,4,8$, respectively; these lines summarize the empirical
trend across the 150 configurations rather than a separate scaling law.

\subsection{Form and fit}
\label{app:law-fit}

\paragraph{Functional form.}
The looping coefficient is
\begin{equation}
 g(T,B)=\exp[c(1-T^{-p}B^{-q})],\qquad
 N_{\mathrm{eff}}=N_{\mathrm{once}}+g(T,B)N_{\mathrm{rec}}.
\end{equation}
The normalization is exact: $g(1,1)=1$, so $(T,B)=(1,1)$ gives
$N_{\mathrm{eff}}=N$. For $c,p,q>0$, $g(T,B)$ increases in $T$ and in $B$ toward $e^c$.
The fitted ceiling of the looping coefficient is $e^c\simeq3.04$;
it is an extrapolation, not an observed asymptotic limit. With full \ac{BPTT} the looping coefficient is
$g(T,T)=\exp[c(1-T^{-(p+q)})]$; for $B<T$ it is smaller by the factor
$\exp[-cT^{-(p+q)}(r^{-q}-1)]$, $r=B/T$, which approaches one as $T$ grows.

\paragraph{Data split.}
We split the 347 scaling-suite configurations into 230 for fitting and 117 for testing.
The fitting set spans fifteen loop settings and includes 226 configurations from the five-size grid, covering
five model sizes at approximately 9, 18, 37, or 74 training tokens per parameter, giving the sizes
a common budget range. Of these 226 configurations, 194 use the eleven settings with at most eight loops
(a few at the largest budgets were not run) and 32 use $T=16$.
The other four fitting configurations are also used to compare a looped model with a non-looped model at each of two
training-compute budgets (\Cref{app:untie}), adding sizes 79M and 162M to the grid's five sizes.
Each configuration enters once through the mean loss of its accepted runs, aggregating
repeated runs and seeds and removing duplicate exports. The other 117 configurations of the scaling suite, mostly at shorter or
longer budgets, enter no fit and serve as separate tests. Predictions use actual consumed tokens and exact
parameter counts. The fixed
normalizations are $N_0=2.4\times10^8$ and $D_0=2^{30}$.

\paragraph{Solver.}
We minimize equally weighted squared residuals in nats per token. All nine
coefficients are refitted on every training fold using 64 independent
uniform random starts, retaining the lowest-cost converged solution.
The trust-region reflective solver uses an analytic Jacobian, Jacobian
scaling, at most 4000 evaluations, and tolerances
$\mathrm{ftol}=\mathrm{xtol}=\mathrm{gtol}=10^{-12}$.
Seeds are 20261010 plus the deterministic fold index. No full-data
coefficients initialize a held-out fold, and no learned preprocessing,
weighting or hyperparameter tuning is used.

In optimization coordinates, $E\in[0,3]$,
$\log A,\log H\in[-10,15]$, $\alpha,\beta,\kappa\in[0.01,3]$,
$p,q\in[0.001,10]$, and
$\log c\in[\log0.001,\log10]$.
Initial ranges are $E\in[0.3,1.2]$, $\log A,\log H\in[0.5,2.5]$,
$\alpha,\beta\in[0.25,0.65]$, $\kappa\in[0.15,0.45]$,
$p,q\in[0.1,2]$, and
$\log c\in[\log0.2,\log2]$. The power-law gain of \Cref{app:law-bounded} uses the same
ranges for the first six coefficients and $\phi_T,\phi_B\in[-5,5]$, initialized in $[0.05,0.8]$.
All 264 fits of each law converge with full-rank Jacobians (the largest-size holdout reaches $E=0$).

\paragraph{Validation and aggregation.}
Residuals are prediction minus measurement. Point LOO holds out one entire
accepted configuration. Size folds hold out all
combinations at the selected model size. Budget folds group the target level
$D_{\mathrm{target}}/D_0$, retaining batch-rounded versions together. Setting folds hold out
an entire $(T,B)$ pair.
Each family covers all 230 configurations. Point LOO trains on 229
configurations and tests on the remaining one, repeating this for all 230.
Table~\ref{tab:law_validation} reports pooled errors, which weight held-out
configurations equally. Every fold converges
and none is dropped. \Cref{app:grid} overlays the law on the 226 fitted
grid configurations; the four fitted configurations from the compute-matched pairs are not shown there.

\begin{table}[!htbp]\centering\small
\caption{Pooled errors in nats/token over all 230 configurations; all 264 fits converge (one at a bound).}\label{tab:law_validation}
\setlength{\tabcolsep}{8.5pt}\begin{tabular}{@{}lrrrrr@{}}\toprule
Evaluation & Fits & RMSE & MAE & Bias & Max. error \\
\midrule
In-sample & 1 & $6.14\times 10^{-3}$ & $4.91\times 10^{-3}$ & $-3.73\times 10^{-16}$ & $2.04\times 10^{-2}$ \\
One configuration out & 230 & $6.40\times 10^{-3}$ & $5.11\times 10^{-3}$ & $-3.30\times 10^{-6}$ & $2.10\times 10^{-2}$ \\
One model size $N$ out & 7 & $7.48\times 10^{-3}$ & $5.89\times 10^{-3}$ & $-2.16\times 10^{-4}$ & $2.58\times 10^{-2}$ \\
One token budget $D$ out & 11 & $6.42\times 10^{-3}$ & $5.12\times 10^{-3}$ & $-3.42\times 10^{-5}$ & $2.08\times 10^{-2}$ \\
One $(T,B)$ setting out & 15 & $8.67\times 10^{-3}$ & $6.74\times 10^{-3}$ & $8.86\times 10^{-4}$ & $2.63\times 10^{-2}$ \\
\bottomrule\end{tabular}\end{table}
\begin{table}[!htbp]\centering\small
\begin{minipage}[t]{0.6\textwidth}\vspace{0pt}\centering\captionsetup{type=table,width=\linewidth}
\caption{The law against the power-law gain of prior work \citep{schwethelm2026how}: $g(T,B)=T^{\phi_T}B^{\phi_B}$ in place of the bounded gain, which reduces to $T^{\phi}$ with full \ac{BPTT}, everything else of \eqref{eq:law} unchanged, fitted on the same configurations with the same folds and solver. Pooled RMSE in nats per token.}\label{tab:law-comparison}
\setlength{\tabcolsep}{6pt}\renewcommand{\arraystretch}{1.55}\begin{tabular}{@{}lrr@{}}\toprule
Evaluation & Bounded gain & Power-law gain \\
\midrule
In-sample & $6.14\times 10^{-3}$ & $1.31\times 10^{-2}$ \\
One configuration out & $6.40\times 10^{-3}$ & $1.36\times 10^{-2}$ \\
One model size $N$ out & $7.48\times 10^{-3}$ & $1.40\times 10^{-2}$ \\
One token budget $D$ out & $6.42\times 10^{-3}$ & $1.35\times 10^{-2}$ \\
One $(T,B)$ setting out & $8.67\times 10^{-3}$ & $1.65\times 10^{-2}$ \\
\bottomrule\end{tabular}
\end{minipage}\hfill
\begin{minipage}[t]{0.36\textwidth}\vspace{0pt}\centering\captionsetup{type=table,width=\linewidth}
\caption{Looping coefficient and total effective size; total size uses recurrent share $f=0.507$.}\label{tab:law_neff}
\setlength{\tabcolsep}{6pt}\begin{tabular}{@{}lrr@{}}\toprule
Setting & $g(T,B)$ & $N_{\mathrm{eff}}/N$ \\
\midrule
$(1,1)$ & 1.00000 & 1.00000 \\
$(2,1)$ & 1.18330 & 1.09294 \\
$(2,2)$ & 1.91023 & 1.46150 \\
$(4,1)$ & 1.36498 & 1.18505 \\
$(4,2)$ & 2.04945 & 1.53209 \\
$(4,4)$ & 2.50361 & 1.76236 \\
$(8,2)$ & 2.17552 & 1.59601 \\
$(8,4)$ & 2.57832 & 1.80024 \\
$(8,8)$ & 2.80330 & 1.91431 \\
$(16,16)$ & 2.93895 & 1.98308 \\
\bottomrule\end{tabular}
\end{minipage}
\end{table}

\subsection{Why a bounded gain}
\label{app:law-bounded}
Prior scaling laws for looped models let the looping gain grow as a
power of the loop count, $T^{\phi}$ \citep{schwethelm2026how}. We extend this form to
$T^{\phi_T}B^{\phi_B}$ in place of $g(T,B)$, keeping the rest of \eqref{eq:law} unchanged.
This form is the limit of our gain as $c\to\infty$ with $cp\to\phi_T$ and $cq\to\phi_B$;
under full \ac{BPTT} it reduces to $T^{\phi}$ with $\phi=\phi_T+\phi_B$.
We fit both laws on the same configurations, folds and solver. The power-law gain is less accurate in every
evaluation in Table~\ref{tab:law-comparison}, with nearly twice the error on a held-out $(T,B)$ setting ($1.65\times10^{-2}$ vs.\ $8.67\times10^{-3}$ nats);
it keeps growing where the runs level off, and held out at sixteen loops its error is three times
that of the bounded gain (\Cref{app:law-t16}). Our gain rises
from $g(1,1)=1$ toward $e^c$; with full \ac{BPTT} and $p+q=1$ it is the reciprocal
exponential gain $\exp[c(1-1/T)]$, and free exponents separate the ceiling
from the rate of approach.

The derivatives distinguish the empirical assumption from its algebraic
consequences. At fixed coefficients, with $s=p+q$, the total derivatives
below follow the full-\ac{BPTT} path $B=T$:
\begin{equation}
\begin{aligned}
 \frac{\partial\log g}{\partial T}&=cpT^{-p-1}B^{-q}>0,\qquad
 \frac{\partial\log g}{\partial B}=cqT^{-p}B^{-q-1}>0,\\
 \frac{\mathrm{d}}{\mathrm{d}T}g(T,T)&=csT^{-s-1}g(T,T)>0,\\
 \frac{\mathrm{d}^2}{\mathrm{d}T^2}g(T,T)&=csT^{-s-2}g(T,T)\,[csT^{-s}-(s+1)].
\end{aligned}
\end{equation}
For the fitted coefficients, $cs<s+1$, so the full-\ac{BPTT} gain is concave for
$T\ge1$. Concavity is not automatic for every positive parameter choice.
The effective size is an equivalence in validation loss, not in stored
knowledge \citep{zhu2025scaling}.

\subsection{Runs at sixteen loops}
\label{app:law-t16}
The fit contains 32 configurations at $T=16$, across 120M, 240M and 480M,
with backward counts 1, 4, 8 and 16 and target token budgets from $2D_0$ to
$8D_0$. Holding out this whole row leaves 198 training configurations and gives
RMSE 0.0115 nats, against 0.0340 for the power-law gain of prior work,
which keeps growing beyond the loop counts it is fitted on. The harder
extrapolation is the row without looping, $T=1$: RMSE 0.0172 (0.0207 for
the power-law gain), the largest forward-row error. The four sixteen-loop
configurations of 480M at $2D_0$ lie outside the fitting set.

\section{Training compute and allocation}
\label{app:compute}

We first define the training-compute cost, then discuss the fixed-size switch points and their experimental checks.
\Cref{app:law-allocation} gives the allocation derivations for prescribed forward loops, joint loop counts,
and fixed or free model size.

\subsection{Training compute}
\label{app:compute-flops}
Training compute counts the FLOPs of the forward and backward passes. We use the usual
convention of 2 FLOPs per parameter and token for the forward pass and 4 for the backward pass
\citep{kaplan2020scaling,hoffmann2022training}. The prelude and coda run forward and backward once per token. Each forward loop runs
the recurrent block forward once, and each backward loop adds a backward pass through it, which costs twice as much. The compute of
the passes per token is
\begin{equation}
  F(T,B) \;\approx\; 6N_{\mathrm{once}}+(2T+4B)\,N_{\mathrm{rec}} + F_{\mathrm{attn}}(T,B) + F_{\mathrm{out}}
  \;=\; 6N_{\mathrm{comp}} + F_{\mathrm{attn}}(T,B) + F_{\mathrm{out}},
  \label{eq:flops}
\end{equation}
where $F_{\mathrm{attn}}$ counts the attention products of every executed layer and $F_{\mathrm{out}}$ the output
projection.\footnote{We use actual layer counts for numerical compute accounting. Each transformer layer has $12d^2$
matrix parameters and input injection has $2d^2$, so the exact matrix contribution under our counting convention is
$72d^2(\ell_p+\ell_c)+4d^2(T+2B)(6\ell_r+1)$. With width $d$, sequence length $S$ and vocabulary size $V$ (32{,}064 with padding):
$F_{\mathrm{attn}}=4Sd\,[\,3(\ell_p+\ell_c)+(T+2B)\,\ell_r\,]$ and $F_{\mathrm{out}}=6Vd$. Embedding lookup, normalization, and activations
are omitted.}
The attention products of the recurrent block follow the same count, one unit per forward loop and two more per backward loop, so
every loop-dependent term grows with $T+2B$, as in the accounting of truncated \ac{BPTT} by \citet{prairie2026parcae} and
\citet{schwethelm2026how}; for a given model, $(2,2)$ and $(4,1)$ therefore cost the same per token, as do $(4,4)$ and $(8,2)$.

For a model of size $N$ we write this compute per token as $F_N(T,B)=u_N+v_N(T+2B)$, where $u_N$ collects the prelude, coda and
output and $v_N>0$ one forward pass through the recurrent block with its attention, so that $D$ training tokens cost
\begin{equation}
  C \;\simeq\; F_N(T,B)\,D .
  \label{eq:law-budget}
\end{equation}
At fixed $N$ and $D$, the compute depends on the loop counts only through $T+2B$. \eqref{eq:law-leading-compute} is its
leading term, without the attention and output costs.

\paragraph{The optimizer update is negligible.}
MuonH also spends FLOPs on each update: five Newton--Schulz iterations per weight matrix \citep{jordan2024muon}, a cost
$F_{\mathrm{opt}}$ per step that depends on the model but not on $T$ or $B$, so that a run of $n_{\mathrm{step}}$ steps costs
\begin{equation}
  C \;=\; F(T,B)\,D + F_{\mathrm{opt}}\,n_{\mathrm{step}} .
  \label{eq:compute}
\end{equation}
Across our runs the update adds 0.06--5.7\% to the compute of the passes. It cannot change the split between forward and backward
loops, since it does not depend on either; leaving it out moves the fixed-size switch points of \Cref{sec:law-fixed-size} by less
than 0.5\%, and it gives the models with more trainable parameters, the untied model and the model without looping at its
compute-optimal size, 0.5--0.75\% more compute, or about 0.001 nats in loss (\Cref{app:untie}). Like prior scaling studies
\citep{schwethelm2026how,prairie2026parcae}, our analysis therefore omits it and uses \eqref{eq:law-budget}; the measured compute in
figures, tables and \Cref{app:setup} includes it unless stated otherwise.

\subsection{Switch points}
\label{app:compute-switch}
At a fixed size, each setting traces a loss curve in $C$, and the best setting at each compute lies on their lower envelope
(Table~\ref{tab:fixed_switch}). Full \ac{BPTT} reaches the envelope almost everywhere: among all fifteen trained settings, four of them
with $T=16$, the best setting at every compute is $(1,1)$, $(2,2)$, $(4,4)$, $(8,8)$ or $(16,16)$, apart from a narrow compute range between $(2,2)$ and $(4,4)$ where $(4,3)$ leads by at most 0.002 nats on our integer menu (\Cref{sec:law-fixed-size}). Truncation lowers the compute per
token but gives up the gain of the loops it does not train, and at equal compute per token a setting that trains all of its fewer
loops is the better buy, as the runs show directly (\Cref{app:law-runs}). Measured in the tokens per parameter that the
model without looping trains on at the same compute, two loops become the best choice at about 14--17, four overtake two at about 312--413,
and eight overtake four at about 7{,}000--9{,}600 (sixteen pay only beyond 147{,}000--210{,}000).

\begin{table}[t]
  \centering
  \caption{Model size fixed: the training compute at which the best loop setting changes, given as the tokens per parameter that
    the model without looping trains on at that compute. Among all fifteen trained settings, four of them with $T=16$, only full
    \ac{BPTT} reaches the frontier, apart from $(4,3)$, which leads briefly between $(2,2)$ and $(4,4)$ by at most 0.002 nats.
    $^\dagger$Beyond the compute of the runs of that size.
    $^\ddagger$Involves $T=16$, trained at three sizes and short budgets only.}
  \label{tab:fixed_switch}
  \small
  \setlength{\tabcolsep}{9pt}
  \begin{tabular}{@{}lccccc@{}}
    \toprule
    Best setting changes & 60M & 120M & 240M & 480M & 1B \\
    \midrule
    $(1,1)\to(2,2)$ & 16 & 17 & 16 & 15 & 14 \\
    $(2,2)\to(4,4)$ & 400 & 413 & 374$^\dagger$ & 345$^\dagger$ & 312$^\dagger$ \\
    $(4,4)\to(8,8)$ & 9{,}622$^\dagger$ & 9{,}648$^\dagger$ & 8{,}605$^\dagger$ & 7{,}833$^\dagger$ & 7{,}025$^\dagger$ \\
    $(8,8)\to(16,16)$ & 210{,}087$^\ddagger$ & 204{,}851$^\ddagger$ & 182{,}247$^\ddagger$ & 164{,}533$^\ddagger$ & 146{,}710$^\ddagger$ \\
    \bottomrule
  \end{tabular}
\end{table}

\subsection{The runs at a fixed size}
\label{app:compute-runs}
The runs confirm the first switch and the role of the trained loops (Figure~\ref{fig:alloc_fixed}; Figure~\ref{fig:teaser}c
previews the 120M slice). Beyond
about 30 tokens per parameter, $(2,2)$ ends below the model without looping at every point we measured, by about 0.01 nats at 50
tokens per parameter and by 0.025 at 200. Loops without gradient do not win: $(4,1)$, which costs as much per token as $(2,2)$ but
trains only one of its four loops, stays above the model without looping up to 200 tokens per parameter. At 60M, 120M and 240M, where the
runs reach several hundred tokens per parameter, $(4,4)$ overtakes $(2,2)$, as the law predicts, but already at about 200 tokens per
parameter rather than at the law's 374--413.
At one compute per size far into this range, that of $(4,4)$ at $32D_0$ for 60M and 240M and at $24D_0$ for 120M (1{,}291, 533
and 335 tokens per parameter of the model without looping), the runs read 2.514, 2.389 and 2.247 for the model without looping,
2.479, 2.358 and 2.223 for $(2,2)$, and 2.460, 2.345 and 2.218 for $(4,4)$: $(4,4)$ is the lowest, 0.019, 0.012 and 0.006 nats below
$(2,2)$, while $(4,3)$ at the same compute trails $(4,4)$ by 0.010, 0.008 and 0.006 and $(8,8)$ lies above it at all three sizes.
At 480M the same compute is reached at 174 tokens per parameter, before the switch: the model without looping reads 2.141, $(2,2)$ 2.121
and $(4,4)$ 2.124, so $(2,2)$ and $(4,4)$ are level and both stay below the model without looping.
The runs with $T=16$, trained at $2D_0$--$8D_0$ (\Cref{app:law-t16}), lie well below the compute where eight or sixteen loops pay and stay above $(4,4)$ at equal compute.

\subsection{Allocation identities and their assumptions}
\label{app:law-allocation}

We follow the allocation regimes of \Crefrange{sec:law-fixed-forward}{sec:law-free-size}:
first prescribe $T$, then optimize the split between $T,B$, and finally
allocate loops, data, and model size. The prescribed-$T$ and free-size
results use the leading matrix-FLOP approximation
\eqref{eq:law-leading-compute}; the joint-loop and fixed-size
calculations use \eqref{eq:law-budget}, with the attention and output
costs. All omit the optimizer update (\Cref{app:compute-flops}).

\subsubsection{Fixed forward-loop count: optimize size, data, and backward loops}
\begin{lemma}[Optimal backward-loop allocation]
\label{lem:law-fixed-forward}
Fix the integer forward-loop count $T\ge1$ and training-compute budget
$C>0$. Suppose the loss follows \eqref{eq:law}, with positive fitted
coefficients $E,A,H,\alpha,\beta,\kappa,c,p,q$ and normalization
scales $N_0,D_0>0$. Let $N=N_{\mathrm{once}}+N_{\mathrm{rec}}$ be the
parameter count. Assume the recurrent block holds a fixed fraction
$\rho\in[0,1]$ of parameters as $N$ varies ($N_{\mathrm{rec}}=\rho N$,
$N_{\mathrm{once}}=(1-\rho)N$).
Under the leading-FLOP approximation $C=6D N_{\mathrm{comp}}$,
for each $B\in\{1,\ldots,T\}$ let $L^*_{T,B}(C)$ be the minimum loss
over $N,D>0$ at budget $C$; then
\begin{equation}
 \operatorname*{arg\,min}_{B\in\{1,\ldots,T\}}L^*_{T,B}(C)
 =\operatorname*{arg\,min}_{B\in\{1,\ldots,T\}}\chi(T,B),
 \qquad \chi(T,B)=\frac{N_{\mathrm{comp}}}{N_{\mathrm{eff}}}.
 \label{eq:law-fixed-forward-opt}
\end{equation}
The minimizing backward counts are independent of $C$.
\end{lemma}

\begin{proof}
The proof reduces the looped allocation problem to a non-looped one.
With a fixed recurrent fraction, the effective and compute-equivalent
parameter counts are both constant multiples of $N$, so the loop setting
enters the budget only through the factor $\chi(T,B)$. We first fix $B$
and solve the size--data allocation in closed form. The optimized loss then
depends on $B$ only through a positive power of $\chi(T,B)$, multiplied by
a factor in $C$ that is common to all backward counts, so comparing backward
counts reduces to comparing $\chi$.

\emph{Step 1: the price of an effective parameter.}
Let us first fix the backward loop count $B$. Substituting
$N_{\mathrm{rec}}=\rho N$ and $N_{\mathrm{once}}=(1-\rho)N$ into the
effective parameter count of \eqref{eq:law} and the compute-equivalent
count of \eqref{eq:law-leading-compute} makes both proportional to $N$:
\begin{equation}
 \frac{N_{\mathrm{eff}}}{N}=1-\rho+\rho g(T,B),\qquad
 \frac{N_{\mathrm{comp}}}{N}=1-\rho+\rho\,\frac{T+2B}{3},\qquad
 \chi(T,B)=\frac{1-\rho+\rho(T+2B)/3}{1-\rho+\rho g(T,B)}>0.
\end{equation}
The factor $N$ cancels in the ratio, so $\chi$ depends on the loop setting
and the recurrent share $\rho$ but not on model size. Intuitively,
$N_{\mathrm{eff}}$ measures what the model is worth in loss and
$N_{\mathrm{comp}}$ what it costs in compute, so $\chi$ is the number of
compute-equivalent parameters paid for each effective parameter. Because
this price is the same at every size, choosing $N>0$ is equivalent to
choosing $N_{\mathrm{eff}}>0$, and we optimize over the normalized
effective size $x=N_{\mathrm{eff}}/N_0>0$ directly.

\emph{Step 2: the budget as a non-looped budget.}
Define $y=D/D_0$ and $\bar C=C/(6N_0D_0)$. Since
$N_{\mathrm{comp}}=\chi N_{\mathrm{eff}}$, the budget
$C=6DN_{\mathrm{comp}}$ becomes
\begin{equation}
 \bar C=\chi xy,
 \qquad\text{equivalently}\qquad
 xy=\frac{\bar C}{\chi}.
\end{equation}
A non-looped model has $\chi=1$ and buys effective size and tokens with
$xy=\bar C$. A loop setting therefore acts like a non-looped model whose
budget is divided by $\chi$: settings with $\chi<1$ stretch the budget and
settings with $\chi>1$ shrink it. Once the size $x$ is chosen, the budget
fixes the token count $y=\bar C/(\chi x)$, so a single free variable remains.

\emph{Step 3: a one-variable objective.}
Write $J=Ax^{-\alpha}+Hy^{-\beta}$ for the bracket in \eqref{eq:law}, so
that $L=E+J^\kappa$. Since $z\mapsto E+z^\kappa$ is strictly increasing for
$z>0$ when $\kappa>0$, minimizing $L$ and minimizing $J$ have the same
minimizers. Eliminating $y$ with the budget gives
\begin{equation}
 J(x)=Ax^{-\alpha}
          +H\left(\frac{\chi x}{\bar C}\right)^{\beta}.
\end{equation}
The first term is the size penalty and falls as the model grows. The second
is the data penalty and rises with $x$, because a larger model leaves fewer
tokens within the same budget. The optimum balances these two effects, which pull the model size in opposite directions.

\emph{Step 4: the balance condition and the optimal allocation.}
Both terms are power laws in $x$, so $\log x$ is the natural variable: in it,
each term is an exponential. Holding $T,B,C$ fixed, write
$X=Ax^{-\alpha}$ and $Y=H(\chi x/\bar C)^\beta$. Then
\begin{equation}
 \frac{\mathrm{d}J}{\mathrm{d}\log x}=-\alpha X+\beta Y,
 \qquad
 \frac{\mathrm{d}^2J}{\mathrm{d}(\log x)^2}
 =\alpha^2X+\beta^2Y>0.
\end{equation}
The objective is strictly convex in $\log x$. It also diverges at both ends,
since $X\to\infty$ as $x\to0$ and $Y\to\infty$ as $x\to\infty$. A strictly
convex function that diverges at both ends has exactly one stationary
point, and that point is its global minimum. The stationarity condition
$\alpha X=\beta Y$ states that, at the optimum, a one-percent increase in
model size lowers the size penalty by exactly as much as the resulting loss
of tokens raises the data penalty. Writing it out,
\begin{equation}
 \alpha A x^{-\alpha}
 =\beta H\left(\frac{\chi}{\bar C}\right)^\beta x^\beta
 \quad\Longrightarrow\quad
 x^{\alpha+\beta}=\frac{\alpha A}{\beta H}
 \left(\frac{\bar C}{\chi}\right)^\beta.
\end{equation}
Taking the positive root and substituting it into $y=\bar C/(\chi x)$ yields
\begin{equation}
\begin{aligned}
 x^*&=\left(\frac{\alpha A}{\beta H}\right)^{1/(\alpha+\beta)}
 \left(\frac{\bar C}{\chi}\right)^{\beta/(\alpha+\beta)},\\
 y^*&=\left(\frac{\beta H}{\alpha A}\right)^{1/(\alpha+\beta)}
 \left(\frac{\bar C}{\chi}\right)^{\alpha/(\alpha+\beta)}.
\end{aligned}
\label{eq:law-free-allocation}
\end{equation}
As Step~2 anticipated, the loop setting enters only through the reduced
budget $\bar C/\chi$: the optimal effective size and token count are those of
a compute-optimal non-looped model trained with $\bar C/\chi$.
The optimal size and token count are recovered as
$N^*=N_0x^*/[1-\rho+\rho g(T,B)]$ and $D^*=D_0y^*$.

\emph{Step 5: the optimal loss.}
At the stationary point $Y^*=(\alpha/\beta)X^*$, so the two penalties stay in
a fixed proportion and $J^*=X^*+Y^*=[(\alpha+\beta)/\beta]X^*$.
Substituting $x^*$ into $X^*=A(x^*)^{-\alpha}$ gives
\begin{equation}
\begin{aligned}
 J^*&=\frac{\alpha+\beta}{\beta}A(x^*)^{-\alpha}\\
 &=\frac{\alpha+\beta}{\beta}
 A^{\beta/(\alpha+\beta)}
 \left(\frac{\beta H}{\alpha}\right)^{\alpha/(\alpha+\beta)}
 \left(\frac{\bar C}{\chi}\right)^{-\alpha\beta/(\alpha+\beta)}.
\end{aligned}
\end{equation}
Define the positive constants
\begin{equation}
 \eta=\frac{\kappa\alpha\beta}{\alpha+\beta},\qquad
 K=\left[\frac{\alpha+\beta}{\beta}
 A^{\beta/(\alpha+\beta)}
 \left(\frac{\beta H}{\alpha}\right)^{\alpha/(\alpha+\beta)}\right]^\kappa.
\end{equation}
Raising $J^*$ to the power $\kappa$ restores the loss:
\begin{equation}
 L^*_{T,B}(C)-E=K\bar C^{-\eta}\chi(T,B)^\eta.
 \label{eq:law-optimized-loss}
\end{equation}
This is the compute-optimal loss of a non-looped model at budget $\bar C/\chi$.

\emph{Step 6: comparing backward counts.}
In \eqref{eq:law-optimized-loss}, $K$ and $\eta$ depend only on the fitted
coefficients and $\bar C^{-\eta}$ only on the budget; none of them depends on
$B$. The backward count enters only through $\chi(T,B)^\eta$, which is
strictly increasing in $\chi$ because $\eta>0$. The optimized losses over
$B\in\{1,\ldots,T\}$ are therefore ordered exactly as the values of $\chi$,
so the two have the same minimizers, including any ties. This proves
\eqref{eq:law-fixed-forward-opt}.
Changing the budget multiplies every optimized excess loss by the same
factor $\bar C^{-\eta}$, which leaves their ranking, and hence the minimizing
backward counts, unchanged whenever the admissible set of $B$ is fixed.
In the limit $\rho=0$ the model has no recurrent block, so $\chi=1$ for every
$B$ and all backward counts tie, as expected with nothing to
truncate.
\end{proof}

For a restricted menu, minimize over its available backward counts.
Figure~\ref{fig:law-fixed-forward} evaluates $\chi$ for every integer $B$ at
$T=2,4,8,16$ (Table~\ref{tab:law-fixed-forward-opt} lists the optima).

\begin{figure}[t]
\centering
\includegraphics[width=0.85\linewidth,keepaspectratio]{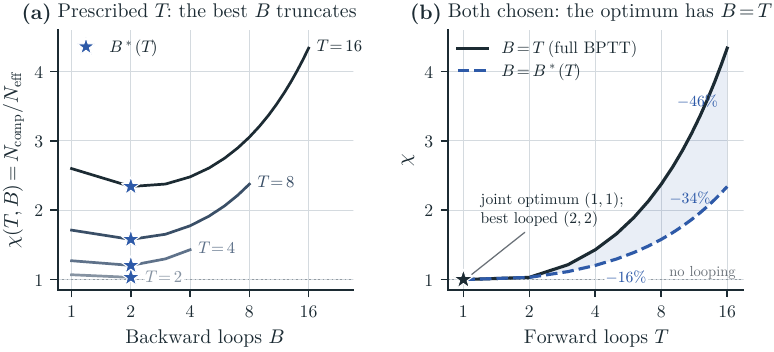}
\caption{\textbf{Truncation can be favored at a prescribed $T$, but not when $T$ and $B$ are chosen jointly.} The law's
$\chi=N_{\mathrm{comp}}/N_{\mathrm{eff}}$; lower $\chi$ means lower compute-optimal loss. (a) Stars: the best $B$ at each $T$.
(b) Percentages: compute saved by the best $B$ over $B=T$ at equal loss.}
\label{fig:law-fixed-forward}
\end{figure}

\paragraph{Separate laws for the contour comparison.}
To check the shared looping coefficient $g(T,B)$, we fit
$L=E_s+[A_s(N/N_0)^{-\alpha_s}+H_s(D/D_0)^{-\beta_s}]^{\kappa_s}$
separately to each setting $s$ at $T=4,8$ in Figure~\ref{fig:iso-study}.
The contours use compute from \eqref{eq:law-leading-compute} with
$N_{\mathrm{rec}}/N\simeq0.507$; their minimum-compute points trace the
compute-optimal frontiers. These are in-sample predictions of the
separate laws.

\paragraph{Choosing the backward-loop count.}
For a continuous relaxation, an interior optimum balances the relative
increase in compute-equivalent size against that in loss-equivalent size.
The ratio is independent of model size, so we can differentiate it
at fixed $N_{\mathrm{once}},N_{\mathrm{rec}}$. By definition,
\begin{equation}
 \frac{\partial N_{\mathrm{comp}}}{\partial B}=\frac23N_{\mathrm{rec}},
 \qquad
 \frac{\partial N_{\mathrm{eff}}}{\partial B}
 =cqT^{-p}B^{-q-1}g(T,B)N_{\mathrm{rec}}.
\end{equation}
Differentiating $\log\chi=\log N_{\mathrm{comp}}-\log N_{\mathrm{eff}}$
and setting the derivative to zero yields
\begin{equation}
 \frac{2B N_{\mathrm{rec}}}{3N_{\mathrm{comp}}}
 =cqT^{-p}B^{-q}\frac{g(T,B)N_{\mathrm{rec}}}{N_{\mathrm{eff}}}.
 \label{eq:law-backward-balance}
\end{equation}
This is an implicit equation for a continuous candidate, as both
parameter counts depend on $B$; boundary values $B=1,T$ are also
checked, and the discrete optimum evaluates $\chi$ at each
integer.

The equation also gives an explicit asymptotic form when $N_{\mathrm{rec}}>0$.
Suppose $0<p<1$ and $q>0$, as in the fitted law, and consider large $T$
with $B\ll T$.
Then $g(T,B)\to e^c$,
$N_{\mathrm{comp}}\sim TN_{\mathrm{rec}}/3$, and
$N_{\mathrm{eff}}\to N_{\mathrm{once}}+e^cN_{\mathrm{rec}}$.
Substituting these limits into \eqref{eq:law-backward-balance}
and rearranging gives
\begin{equation}
 B^*(T)\sim
 \left[\frac{cq}{2}
 \frac{e^cN_{\mathrm{rec}}}{N_{\mathrm{once}}+e^cN_{\mathrm{rec}}}
 \right]^{1/(1+q)}T^{(1-p)/(1+q)}.
\end{equation}
The exponent lies between zero and one, so this solution grows while
$B^*/T\to0$, matching the assumed limit. This is a large-$T$
prediction; finite loop counts require discrete optimization.

\subsubsection{Joint forward--backward allocation at equal compute}
\label{app:law-joint-loops}
\begin{lemma}[Optimal joint forward--backward allocation]
\label{lem:law-joint-loops}
Fix an architecture with $N_{\mathrm{rec}}>0$ and a token count $D>0$.
Suppose the loss follows \eqref{eq:law}, with positive fitted
coefficients $E,A,H,\alpha,\beta,\kappa$, $c,p,q$ and positive normalization
scales $N_0,D_0$. Let $T$ and $B$ be continuous variables. For any
$S\geq3$, consider the equal-compute problem under \eqref{eq:law-budget}:
\begin{equation}
 \underset{T,B\in\mathbb R}{\operatorname{minimize}}\quad L(N,D,T,B)
 \qquad \text{subject to}\quad T+2B=S,\qquad 1\leq B\leq T.
\end{equation}
If $q\geq2p$, its unique solution is $T^*=B^*=S/3$, corresponding to
full BPTT.
\end{lemma}

\begin{proof}
With the token count and architecture fixed, every split on the line
$T+2B=S$ has the same compute, so only the looping coefficient $g(T,B)$
varies across feasible splits. We therefore maximize $g$ along this line,
reduce the problem to a strictly convex function of $B$ alone, and show that
its derivative is negative on the whole feasible interval when $q\geq2p$.
The minimum then lies at the boundary $B=T$, which is full BPTT.

At fixed architecture and token count, \eqref{eq:law-budget} depends on
the loop counts only through $S=T+2B$, including the attention products.
Thus every feasible split has the same compute. Since the loss is
strictly decreasing in $N_{\mathrm{eff}}$ and $N_{\mathrm{rec}}>0$,
minimizing it is equivalent to maximizing $g(T,B)$.
Using $g(T,B)=\exp[c(1-T^{-p}B^{-q})]$ with $c>0$, this is equivalent
to minimizing $T^{-p}B^{-q}$, or its logarithm.
Substituting $T=S-2B$ gives $1\leq B\leq S/3$ and the strictly convex
objective
\begin{equation}
 h(B)=-p\log(S-2B)-q\log B,\qquad
 h''(B)=\frac{4p}{(S-2B)^2}+\frac{q}{B^2}>0.
\end{equation}
When $q\geq2p$, the constraint $B\leq T$ gives
\begin{equation}
 h'(B)=\frac{2p}{T}-\frac qB
 \leq\frac{2p-q}{T}\leq0.
\end{equation}
The derivative is strictly negative whenever $B<T$, including when
$q=2p$. Hence the unique minimum is at $B=S/3$, where $T=B$.
For $S=3$, this is also the only feasible split.
\end{proof}

The replacement preserves $N,D$ and strictly improves any truncated
setting, so the conclusion also holds when either or both are optimized,
provided the replacement is admissible. It does not require fixed
parameter proportions across model sizes. Including the optimizer update
does not change the result, since its cost is independent of $T,B$.

The condition is stronger than $q>p$: a backward loop costs twice a
forward loop, so the comparison is $q$ versus $2p$.
Our adopted fit satisfies $q>2p$. For integer or restricted loop menus,
$S/3$ may be unavailable, so the continuous dominance argument does not
replace the discrete check \eqref{eq:law-fixed-opt}.

\subsubsection{Fixed model size: allocate loops and data}
\label{app:law-fixed-size}
With $N$ fixed, size cannot rebalance against data, and the loop setting
determines how many tokens the compute buys:
$D_N(C;T,B)=C/F_N(T,B)$ by \eqref{eq:law-budget}.
\Cref{sec:law-fixed-size}
compares the fifteen observed settings $\mathcal S$ at their respective
token budgets:
\begin{equation}
 (T^*,B^*)\in\operatorname*{arg\,min}_{(T,B)\in\mathcal S}
 L\bigl(N,D_N(C;T,B),T,B\bigr),
 \qquad D^*=D_N(C;T^*,B^*).
 \label{eq:law-fixed-opt}
\end{equation}
The switch points quoted in \Cref{sec:law-fixed-size} are computed with
the measured compute, optimizer update included; leaving it out moves them
by less than 0.5\% (\Cref{app:compute-flops}).

For continuous loop counts, Lemma~\ref{lem:law-joint-loops} establishes
the preference for full \ac{BPTT} when $q\geq2p$, as in our adopted fit.
We therefore set $B=T$ and first characterize the interior balance
between loops and tokens, then derive its large-compute limit.

\begin{lemma}[Optimal loop--token balance at fixed model size]
\label{lem:law-fixed-balance}
Fix an architecture with $N_{\mathrm{once}}\geq0$ and $N_{\mathrm{rec}}>0$,
and a compute budget $C>0$. Suppose the loss follows \eqref{eq:law},
with positive fitted coefficients $E,A,H,\alpha,\beta,\kappa,c,p,q$
and positive normalization scales $N_0,D_0$. Let $B=T\geq1$
be continuous and use \eqref{eq:law-budget}, so that
$F_N(T,T)=u_N+3v_NT$ with fixed $u_N,v_N>0$ and
$D=C/F_N(T,T)$. Write $s=p+q$ and define
\begin{equation}
\begin{aligned}
 N_{\mathrm{eff}}(T)&=N_{\mathrm{once}}
       +N_{\mathrm{rec}}\exp[c(1-T^{-s})],\\
 X(T)&=A\left(\frac{N_{\mathrm{eff}}(T)}{N_0}\right)^{-\alpha},
 \qquad Y(T)=H\left(\frac{D}{D_0}\right)^{-\beta}.
\end{aligned}
\end{equation}
Throughout this derivation, a prime denotes a total derivative with
respect to $T$ along $B=T$, holding the architecture and $C$ fixed.
In particular,
\begin{equation}
\begin{aligned}
 N_{\mathrm{eff}}'(T)
   &\equiv\frac{\mathrm d N_{\mathrm{eff}}(T)}{\mathrm dT}
    =N_{\mathrm{rec}}cs\,g(T,T)T^{-s-1},\\
 F_N'(T,T)&\equiv\frac{\mathrm d}{\mathrm dT}F_N(T,T)=3v_N.
\end{aligned}
\end{equation}
Thus $N_{\mathrm{eff}}'$ measures how effective size changes with loop
count ($N$ is fixed), and
$F_N'(T,T)$ includes both loop counts as $B=T$ grows.
Every interior optimum satisfies
\begin{equation}
 \alpha X\frac{N_{\mathrm{eff}}'}{N_{\mathrm{eff}}}
 =\beta Y\frac{F_N'(T,T)}{F_N(T,T)}.
 \label{eq:law-marginal}
\end{equation}
Equivalently, the stationary points are parameterized by
\begin{equation}
 C(T)=D_0\left[
 \frac{3\beta Hv_N}{\alpha A}
 F_N(T,T)^{\beta-1}
 \left(\frac{N_{\mathrm{eff}}}{N_0}\right)^\alpha
 \frac{N_{\mathrm{eff}}}{N_{\mathrm{eff}}'}
 \right]^{1/\beta},
 \qquad D(T)=\frac{C(T)}{F_N(T,T)}.
 \label{eq:law-full-branch}
\end{equation}
\end{lemma}

\begin{proof}
At fixed model size and compute, the loop count trades effective size
against tokens: raising $T=B$ increases $N_{\mathrm{eff}}$ but, through the
budget, reduces $D$. We write the loss as a function of $T$ alone by
substituting $D=C/F_N(T,T)$, differentiate its size and data terms, and set
the total derivative to zero, which gives \eqref{eq:law-marginal}. This
first-order condition involves $C$ only through the factor $C^{-\beta}$, so
it can be solved explicitly for the compute at which each $T$ is stationary,
which gives \eqref{eq:law-full-branch}.

Since $\kappa>0$, minimizing loss is equivalent to minimizing
$J$; substituting the budget gives
\begin{equation}
 J(T;N,C)=X(T)+Y(T)
       =A\left(\frac{N_{\mathrm{eff}}}{N_0}\right)^{-\alpha}
       +H\left(\frac{D_0F_N(T,T)}{C}\right)^{\beta}.
\end{equation}
Increasing $T$ raises the effective size and lowers $X$, but reduces
the token budget and raises $Y$. At fixed $N,C$, differentiating the
token budget gives
\begin{equation}
 D'\equiv\frac{\mathrm d}{\mathrm dT}
       \left(\frac{C}{F_N(T,T)}\right)
   =-\frac{C F_N'(T,T)}{F_N(T,T)^2},
 \qquad \frac{D'}{D}=-\frac{F_N'(T,T)}{F_N(T,T)}.
\end{equation}
Consequently, the derivatives of the two loss terms are
\begin{equation}
 X'\equiv\frac{\mathrm dX}{\mathrm dT}
   =-\alpha X\frac{N_{\mathrm{eff}}'}{N_{\mathrm{eff}}},
 \qquad Y'\equiv\frac{\mathrm dY}{\mathrm dT}
   =\beta Y\frac{F_N'(T,T)}{F_N(T,T)}.
\end{equation}
Setting $J'=X'+Y'=0$ gives
\eqref{eq:law-marginal}, and
$N_{\mathrm{eff}}'=N_{\mathrm{rec}}cs\,g(T,T)T^{-s-1}>0$
with $F_N'=3v_N$ yields
\begin{equation}
 \alpha A\left(\frac{N_{\mathrm{eff}}}{N_0}\right)^{-\alpha}
 \frac{N_{\mathrm{eff}}'}{N_{\mathrm{eff}}}
 =3\beta Hv_N\left(\frac{D_0}{C}\right)^\beta
 F_N(T,T)^{\beta-1}.
\end{equation}
All factors are positive, so solving for $C$ gives
\eqref{eq:law-full-branch}.
\end{proof}

This is a necessary condition for an interior optimum; boundary values
and all allowed integer settings must also be compared.
On the finite menu, added compute goes to tokens until a deeper setting
reaches a lower predicted loss, at which point the token budget drops.

\begin{lemma}[Asymptotic loop--token allocation]
\label{lem:law-fixed-asymptotic}
Suppose the loss follows \eqref{eq:law}, with positive fitted
coefficients $E,A,H,\alpha,\beta,\kappa,c,p,q$ and positive normalization
scales $N_0,D_0$. Use the fixed architecture and compute model of
Lemma~\ref{lem:law-fixed-balance}, and allow all continuous $T\geq1$
with no upper bound. Let $T^*(C)$ be any global
minimizer along $B=T$, and let $D^*(C)=C/F_N(T^*,T^*)$.
As $C\to\infty$, there are constants $K_T,K_D>0$, independent of $C$,
such that
\begin{equation}
 T^*(C)\sim K_T C^{\frac{\beta}{\beta+p+q}},
 \qquad
 D^*(C)\sim K_D C^{\frac{p+q}{\beta+p+q}}.
 \label{eq:law-fixed-asymptotic}
\end{equation}
\end{lemma}

\begin{proof}
The proof has two steps. First, we show that a global minimizer exists
for every budget and grows without bound as compute increases, so
at large compute it is interior and satisfies the stationarity relation
\eqref{eq:law-full-branch} of Lemma~\ref{lem:law-fixed-balance}. Second, we
take the large-$T$ limit of that relation, where $C(T)$ is
asymptotically a power of $T$; inverting it gives $T^*(C)$, and the
budget gives $D^*(C)$.

For each finite $C$, the continuous objective $J$ tends to infinity
as $T\to\infty$, since its data term grows without bound, so a
global minimizer exists. Moreover, $T^*(C)\to\infty$ as $C\to\infty$:
for any fixed $M\geq1$, every $T\leq M$ has $J(T;N,C)\geq X(M)$,
whereas $J(2M;N,C)\to X(2M)<X(M)$. No minimizer stays in
$[1,M]$ at large compute; it is eventually
interior and satisfies \eqref{eq:law-full-branch}.

Writing $s=p+q$, the large-$T$ limits are
\begin{equation}
\begin{aligned}
 N_{\mathrm{eff}}&\longrightarrow N_{\mathrm{once}}+e^cN_{\mathrm{rec}},\\
 N_{\mathrm{eff}}'&\sim N_{\mathrm{rec}}cs e^c T^{-s-1},
 \qquad F_N(T,T)\sim3v_NT.
\end{aligned}
\end{equation}
The bracket in \eqref{eq:law-full-branch} is thus asymptotic
to a positive constant times $T^{\beta-1}T^{s+1}=T^{\beta+s}$,
so $C(T)$ is asymptotic to a positive constant times
$T^{(\beta+s)/\beta}$. Evaluating this relation at $T^*(C)$ and
inverting gives the first part of \eqref{eq:law-fixed-asymptotic},
and $D^*(C)\sim C/(3v_NT^*(C))$ gives the second,
recovering \eqref{eq:law-branch-powers}.
\end{proof}

The current coefficients give exponents approximately 0.238 and 0.762
(\eqref{eq:alloc-fixed-scaling} rounds them to 0.24 and 0.76).
These are asymptotic consequences of the fitted law along $B=T$, not a
refit to measured optima. The dominance argument above applies within
this law and compute model; its extrapolation to unbounded depth is not
empirically established. Finite-depth overheads and allowed loop counts
determine the practical switches.

\subsubsection{Free model size: compare joint allocations}
\label{app:law-free-size}
This subsection supports the comparisons of \Cref{sec:law-free-size} in two
steps. Lemma~\ref{lem:law-untie} compares a fully trained looped model with
the non-looped model that unties its loops at the same leading compute.
Corollary~\ref{cor:law-frontier} then lets the model size adapt to the
budget and compares the compute-optimal frontiers of all loop settings.

\begin{lemma}[Untied copies outweigh full-BPTT loops]
\label{lem:law-untie}
Suppose the looping coefficient is $g(T,B)=\exp[c(1-T^{-p}B^{-q})]$ with
$c,p,q>0$, and write $s=p+q$. If
\begin{equation}
 g(2,2)<2
 \qquad\text{and}\qquad
 cs\,2^{-s}<1,
 \label{eq:law-untie-conditions}
\end{equation}
then $g(T,T)<T$ for every real $T\geq2$. Consequently, for an
architecture with $N_{\mathrm{rec}}>0$ and a fixed token count, replacing a
fully trained $T$-loop model, $T\geq2$, with a non-looped model of size
$N_{\mathrm{once}}+TN_{\mathrm{rec}}$ preserves the leading compute
\eqref{eq:law-leading-compute} and increases the effective size by
$[T-g(T,T)]N_{\mathrm{rec}}>0$.
\end{lemma}

\begin{proof}
Each untied copy of the recurrent block adds one full block of effective
parameters, whereas each additional trained loop adds less than one once
the gain begins to saturate. We make this precise by following the ratio
$g(T,T)/T$, the effective size contributed per executed copy of the block.
The first condition in \eqref{eq:law-untie-conditions} places this ratio
below one at $T=2$, and the second makes it decrease for all $T\geq2$, so it
stays below one. The statement about effective size then follows by
comparing the two parameter counts term by term.

Along full BPTT, $\log g(T,T)=c(1-T^{-s})$, so
\begin{equation}
 \frac{\mathrm d}{\mathrm dT}\log\frac{g(T,T)}{T}
 =csT^{-s-1}-\frac1T
 =\frac{csT^{-s}-1}{T}.
\end{equation}
Because $s>0$, the factor $T^{-s}$ decreases in $T$, so for every
$T\geq2$ we have $csT^{-s}\leq cs\,2^{-s}<1$ and the derivative is
negative. The ratio $g(T,T)/T$ is therefore strictly decreasing on
$[2,\infty)$, and
\begin{equation}
 \frac{g(T,T)}{T}\leq\frac{g(2,2)}{2}<1
 \qquad(T\geq2).
\end{equation}
For the second statement, the leading compute-equivalent size of the
looped model is $N_{\mathrm{once}}+(T+2T)N_{\mathrm{rec}}/3
=N_{\mathrm{once}}+TN_{\mathrm{rec}}$ by \eqref{eq:law-leading-compute},
which equals the parameter count, and hence the leading compute, of the
non-looped replacement. Their effective sizes are
$N_{\mathrm{once}}+g(T,T)N_{\mathrm{rec}}$ and
$N_{\mathrm{once}}+TN_{\mathrm{rec}}$, which differ by
$[T-g(T,T)]N_{\mathrm{rec}}$, positive by the first statement and
$N_{\mathrm{rec}}>0$.
\end{proof}

The adopted coefficients satisfy both conditions:
$g(2,2)\simeq1.910<2$ and $cs\,2^{-s}\simeq0.585<1$. The restriction to
$T\geq2$ is needed, because the continuous extension has $g(T,T)>T$ just
above $T=1$, where the derivative above is positive ($cs\simeq1.400>1$).

\begin{corollary}[Free-size frontiers]
\label{cor:law-frontier}
Adopt the assumptions of Lemma~\ref{lem:law-fixed-forward}: the loss law
\eqref{eq:law}, a fixed recurrent fraction $\rho$, and the leading-FLOP
budget $C=6DN_{\mathrm{comp}}$. For each loop setting $(T,B)$ and budget
$C>0$, let $L^*_{T,B}(C)$ be the minimum loss over continuous $N,D>0$, with
minimizers $N^*_{T,B}(C)$ and $D^*_{T,B}(C)$, and let
$\eta=\kappa\alpha\beta/(\alpha+\beta)$. Then:
\begin{enumerate}[label=(\roman*)]
\item \emph{Frontier identity.} Every loop setting behaves like a
non-looped model with budget $C/\chi(T,B)$:
\begin{equation}
\begin{aligned}
 L^*_{T,B}(C)&=L^*_{1,1}\!\left(\frac C{\chi(T,B)}\right),\\
 N^*_{T,B}(C)&=\left(\frac{N}{N_{\mathrm{eff}}}\right)_{T,B}
 N^*_{1,1}\!\left(\frac C{\chi(T,B)}\right),\\
 D^*_{T,B}(C)&=D^*_{1,1}\!\left(\frac C{\chi(T,B)}\right),
\end{aligned}
\label{eq:law-frontier-identity}
\end{equation}
where the ratio $(N/N_{\mathrm{eff}})_{T,B}=1/[1-\rho+\rho g(T,B)]$ is
evaluated at the indicated loop setting and does not depend on $N$.
\item \emph{Compute premium.} If the loop setting receives $\lambda C$, with
$\lambda>0$, while the non-looped model receives $C$, then
\begin{equation}
 \frac{L^*_{T,B}(\lambda C)-E}{L^*_{1,1}(C)-E}
 =\left(\frac{\lambda}{\chi(T,B)}\right)^{-\eta},
 \qquad
 \frac{N^*_{T,B}(\lambda C)}{N^*_{1,1}(C)}
 =\left(\frac{N}{N_{\mathrm{eff}}}\right)_{T,B}
 \left(\frac{\lambda}{\chi(T,B)}\right)^{\beta/(\alpha+\beta)}.
 \label{eq:law-premium}
\end{equation}
In particular, the loop setting reaches a lower loss exactly when
$\lambda>\chi(T,B)$.
\item \emph{Equal-compute gap.} At $\lambda=1$,
\begin{equation}
 L^*_{T,B}(C)-L^*_{1,1}(C)
 =\bigl[L^*_{1,1}(C)-E\bigr]
   \bigl[\chi(T,B)^{\eta}-1\bigr].
 \label{eq:law-gap}
\end{equation}
\end{enumerate}
\end{corollary}

\begin{proof}
The proof of Lemma~\ref{lem:law-fixed-forward} already contains the key
fact: for a fixed loop setting, the budget constraint reads
$C/\chi=6N_{\mathrm{eff}}D$, so optimizing a looped model over its size
and tokens is the non-looped problem with a reduced budget. Part~(i)
restates it for the non-looped optimum, and parts~(ii)
and~(iii) follow by comparing the optimal losses and sizes,
\eqref{eq:law-optimized-loss} and \eqref{eq:law-free-allocation}, at two
budgets.

\emph{Part (i).} By Step~1 of the proof of
Lemma~\ref{lem:law-fixed-forward}, $N_{\mathrm{eff}}$ is the fixed
multiple $1-\rho+\rho g(T,B)$ of $N$, so optimizing over $N$ is equivalent
to optimizing over $N_{\mathrm{eff}}$, and $N_{\mathrm{comp}}=\chi
N_{\mathrm{eff}}$ turns the budget into $C/\chi=6N_{\mathrm{eff}}D$.
In the variables $(N_{\mathrm{eff}},D)$, the loss \eqref{eq:law} is the
non-looped law, since $(1,1)$ has $g(1,1)=1$ and $N_{\mathrm{eff}}=N$, and
the constraint is the non-looped budget $C/\chi$, since $\chi(1,1)=1$. The
two problems therefore have the same minimum and the same unique minimizers
(Step~4 of that proof), namely $L^*_{1,1}(C/\chi)$, with effective size
$N^*_{1,1}(C/\chi)$ and token count $D^*_{1,1}(C/\chi)$. Dividing the
effective size by $1-\rho+\rho g(T,B)$ recovers $N^*_{T,B}(C)$, which
gives \eqref{eq:law-frontier-identity}.

\emph{Part (ii).} By \eqref{eq:law-optimized-loss}, the excess loss of
$(T,B)$ at budget $\lambda C$ is $K\bar C^{-\eta}\lambda^{-\eta}\chi^\eta$,
and that of $(1,1)$ at budget $C$ is $K\bar C^{-\eta}$, because
$\chi(1,1)=1$. Their ratio is $(\lambda/\chi)^{-\eta}$. Since $\eta>0$, it
is below one exactly when $\lambda>\chi$. For the sizes, part~(i) gives
$N^*_{T,B}(\lambda C)=(N/N_{\mathrm{eff}})_{T,B}N^*_{1,1}(\lambda C/\chi)$,
and \eqref{eq:law-free-allocation} with $\chi=1$ shows that
$N^*_{1,1}$ scales as the power $\beta/(\alpha+\beta)$ of the budget, so
$N^*_{1,1}(\lambda C/\chi)/N^*_{1,1}(C)=(\lambda/\chi)^{\beta/(\alpha+\beta)}$.

\emph{Part (iii).} Setting $\lambda=1$ in the loss ratio of part~(ii) gives
$L^*_{T,B}(C)-E=[L^*_{1,1}(C)-E]\chi^\eta$. Subtracting
$L^*_{1,1}(C)-E$ from both sides yields \eqref{eq:law-gap}.
\end{proof}

The identities are exact within the stated approximation and have two
consequences. First, by \eqref{eq:law-free-allocation} every loop setting
has $N^*\propto C^{\beta/(\alpha+\beta)}$ and
$D^*\propto C^{\alpha/(\alpha+\beta)}$ (exponents 0.562 and 0.438 with the
fitted coefficients), so the optimal pair minimizes $\chi(T,B)$ at every
compute. Second, among the fifteen settings $(1,1)$ minimizes $\chi$ and
$(2,2)$ is the best looped setting
(Table~\ref{tab:law-fixed-forward-opt}; Figure~\ref{fig:law-fixed-forward}b).
For $(2,2)$, $N_{\mathrm{eff}}\simeq1.462N$, $N_{\mathrm{comp}}\simeq1.507N$
and $\chi\simeq1.031$, so \eqref{eq:law-frontier-identity} predicts equal
loss at 3.1\% more compute and 31.6\% fewer parameters than $(1,1)$.
By \eqref{eq:law-premium}, 10\% more compute predicts lower loss with
about 29.0\% fewer non-embedding parameters. At equal compute, with
$\eta\simeq0.0559$, the bracketed multiplier in \eqref{eq:law-gap} is
about $1.72\times10^{-3}$, or 0.172\% of the non-looped model's excess loss
above $E$, rather than a constant gap in nats, making the compute
convention consequential.

\paragraph{Limits of the scale-independent ratios.}
The identities describe an empirical loss law under a specified compute
approximation; they do not establish a universal property of looped models.
Size or data bounds can put an optimum on a boundary. A size-dependent
recurrent share or the attention, output, and optimizer terms breaks the
simple $6\chi N_{\mathrm{eff}}D$ factorization. Full-compute optimization therefore
requires an explicit architectural scaling family and a fresh numerical
envelope. Its results should be reported separately from the analytic
ratios, using the same adopted fit.

\section{Compute-optimal pairs and the untied model}
\label{app:untie}

The main text (\Cref{sec:untie}) compares $(2,2)$ with a compute-optimal model without looping while allowing model size and
token count to vary. Below we give the full table for that comparison, followed by a controlled comparison at fixed
tokens against an untied model with a separate block per loop.

\subsection{Compute-optimal pairs}
\label{app:untie-pairs}
We trained $(1,1)$ and $(2,2)$ each at the size and token count that
minimize its loss for the same training compute, with three seeds each
(Table~\ref{tab:compute-optimal-pairs}). The first two pairs also belong to the scaling suite and its fitting set.
At all three budgets, the looped model has 27--32\% fewer non-embedding
parameters and comes within $(1\text{--}8)\times10^{-3}$ nats of the model without looping. At $2.1\times10^{19}$ FLOPs
($2.0\times10^{19}$ for the forward and backward passes alone), training the 344M $(2,2)$ model on 8\% more tokens, and so
with 8\% more compute, lowers its loss to 2.3590, $(7.3\pm1.8)\times10^{-3}$ nats below the 471M model without looping
(Figure~\ref{fig:teaser}d).

\begin{table}[h]
\centering
\caption{\textbf{Compute-optimal pairs.} Each model at the size that minimizes
its loss for the compute; loss: mean of three seeds; gap: $(2,2)$ minus
$(1,1)$, $\pm$ standard error; parameters: non-embedding counts (the models without looping are the 120M, 240M and 480M models of Table~\ref{tab:models}). Compute includes attention, output, and optimizer updates \eqref{eq:compute}; the first three rows are matched in the compute of the
forward and backward passes ($1.5\times10^{18}$, $5.9\times10^{18}$, and $2.0\times10^{19}$ FLOPs, or $9.6\times10^{17}$, $4.1\times10^{18}$, and $1.5\times10^{19}$ under the leading-FLOP count $6N_{\mathrm{comp}}D$ of Figure~\ref{fig:alloc-scaling}d; counting the optimizer update, $(1,1)$ spends
0.7\% more, which corresponds to about $10^{-3}$ nats); the last row trains the 344M $(2,2)$ model on 8\% more tokens (+8\% compute) against the $2.1\times10^{19}$ FLOP $(1,1)$ reference.}
\label{tab:compute-optimal-pairs}
\small
\setlength{\tabcolsep}{8.5pt}
\begin{tabular}{@{}rrrrrr@{}}
\toprule
& \multicolumn{2}{c}{$(1,1)$} & \multicolumn{2}{c}{$(2,2)$} & \\
\cmidrule(lr){2-3}\cmidrule(lr){4-5}
Compute (FLOPs) & Parameters & Loss & Parameters & Loss & Gap \\
\midrule
$1.6\times10^{18}$ & 115M & 2.7098 & 79M & 2.7177 & $(7.9\pm1.9)\times10^{-3}$ \\
$6.0\times10^{18}$ & 239M & 2.5215 & 162M & 2.5253 & $(3.8\pm1.8)\times10^{-3}$ \\
$2.1\times10^{19}$ & 471M & 2.3663 & 344M & 2.3677 & $(1.4\pm1.7)\times10^{-3}$ \\
$2.2\times10^{19}$ (+8\%) & 471M & 2.3663 & 344M & 2.3590 & $(-7.3\pm1.8)\times10^{-3}$ \\
\bottomrule
\end{tabular}
\end{table}

\subsection{The untied model}
\label{app:untie-model}
The untied model of a looped model is the same network in which each loop
has its own copy of the recurrent block, input injection and normalization.
For two loops it holds $N_{\mathrm{once}}+2N_{\mathrm{rec}}$ parameters, about
1.5 times as many, at the same compute per token. Tying layers to save
parameters goes back to ALBERT \citep{lan2020albert}, and \citet{bae2025relaxed}
relax the tie of looped models converted from pretrained ones. Each copy is
initialized like the looped model; with all copies set equal, the untied
model reproduces the looped model's loss and logits exactly.

\paragraph{Design and results.}
Each size trains at about 18 tokens per parameter of the looped model, close
to the compute-optimal ratio of \citet{hoffmann2022training}: 60M at $D_0$,
120M at $2D_0$, 240M at $4D_0$ and 480M at $8D_0$, with the recipe of its
$(2,2)$ run and three seeds per model. Alongside the untied model, we train
$(2,2)$ at the same width and $(2,2)$ widened to about 80\% of the untied
model's non-embedding parameters, on the same tokens (Table~\ref{tab:untie}). At the
untied model's width the looped model trails by $9\times10^{-3}$--$1.7\times10^{-2}$ nats ($1.3\times10^{-2}$
averaged over the four sizes), against $(3\text{--}4)\times10^{-3}$ by the law, whose error
at these $(2,2)$ runs is at most $4\times10^{-3}$ nats. The wider looped model uses 15--22\% more forward-and-backward
compute per token than the untied model and reaches a lower loss by $5\times10^{-3}$--$1.7\times10^{-2}$ nats, in all twelve seed pairs.

\begin{table}[t]
  \centering
  \caption{\textbf{Comparison with the untied model.} At each size: the untied model, $(2,2)$ at its width, and $(2,2)$ widened to about 80\% of its non-embedding parameters. Width with heads $\times$ head dimension; parameters: non-embedding; compute: compute per token of the forward and backward passes relative to the untied model, the same ratio as the training compute since all models of a size train on the same tokens; loss: mean of three seeds; $-$\,untied: paired by seed, mean $\pm$ standard error.}
  \label{tab:untie}
  \small
  \setlength{\tabcolsep}{8pt}
  \begin{tabular}{@{}llrrrrr@{}}
    \toprule
    Size & Model & Width & Parameters & Compute & Loss & $-$\,untied \\
    \midrule
    60M & untied & 768 (6$\times$128) & 87M & 1.00 & 2.7843 &  \\
     & $(2,2)$ & 768 (6$\times$128) & 58M & 1.00 & 2.7931 & $(+8.8\pm3.0)\times10^{-3}$ \\
     & $(2,2)$, wider & 840 (7$\times$120) & 69M & 1.15 & 2.7668 & $(-1.75\pm0.19)\times10^{-2}$ \\
    \addlinespace
    120M & untied & 1024 (8$\times$128) & 180M & 1.00 & 2.5809 &  \\
     & $(2,2)$ & 1024 (8$\times$128) & 115M & 1.00 & 2.5982 & $(+1.73\pm0.14)\times10^{-2}$ \\
     & $(2,2)$, wider & 1152 (9$\times$128) & 146M & 1.22 & 2.5681 & $(-1.28\pm0.23)\times10^{-2}$ \\
    \addlinespace
    240M & untied & 1280 (10$\times$128) & 360M & 1.00 & 2.4147 &  \\
     & $(2,2)$ & 1280 (10$\times$128) & 239M & 1.00 & 2.4272 & $(+1.25\pm0.08)\times10^{-2}$ \\
     & $(2,2)$, wider & 1408 (11$\times$128) & 289M & 1.18 & 2.4044 & $(-1.03\pm0.17)\times10^{-2}$ \\
    \addlinespace
    480M & untied & 1664 (13$\times$128) & 709M & 1.00 & 2.2565 &  \\
     & $(2,2)$ & 1664 (13$\times$128) & 471M & 1.00 & 2.2700 & $(+1.35\pm0.19)\times10^{-2}$ \\
     & $(2,2)$, wider & 1824 (12$\times$152) & 566M & 1.18 & 2.2511 & $(-5.4\pm1.5)\times10^{-3}$ \\
    \bottomrule
  \end{tabular}
\end{table}

\section{Learning rate and batch size}
\label{app:lrbs}

Our runs follow one learning-rate rule and one batch-size rule, both set by the token budget $D$ (Table~\ref{tab:recipe}).
We test each rule against sweeps that vary only the learning rate or only the batch; a learning-rate sweep scales the MuonH and
AdamW learning rates by one common factor. We call a value acceptable when its loss is within 0.25\% of the optimum of its
sweep, about a tenth of what $(4,4)$ gains over $(1,1)$ at the same model size and budget.

\begin{table}[h]
  \centering
  \caption{Batch size, number of steps and peak MuonH learning rate of the hidden weight matrices at each budget. The learning rate is
  \eqref{eq:lr-transfer} with $\gamma=0.40$, with its ratio to the $\gamma=0.32$ value rounded to a power of $2^{1/8}$;
  the AdamW learning rates are multiplied by the same ratio. Warmup takes 2\% of the steps.}
  \label{tab:recipe}
  \small
  \setlength{\tabcolsep}{8.5pt}
  \begin{tabular}{@{}lccccccc@{}}
    \toprule
    Budget $D/D_0$ & $\tfrac12$ & 1 & 2 & 4 & 8 & 16 & 32 \\
    \midrule
    Batch (sequences of 2{,}048 tokens) & 48 & 64 & 88 & 128 & 176 & 256 & 352 \\
    Steps & 5{,}464 & 8{,}192 & 11{,}920 & 16{,}384 & 23{,}832 & 32{,}768 & 47{,}664 \\
    Peak learning rate ($\times10^{-3}$) & 10.7 & 9.25 & 8.11 & 8.01 & 7.03 & 6.36 & 6.09 \\
    \bottomrule
  \end{tabular}
\end{table}

\begin{figure}[tp]
  \centering
  \includegraphics[width=\textwidth]{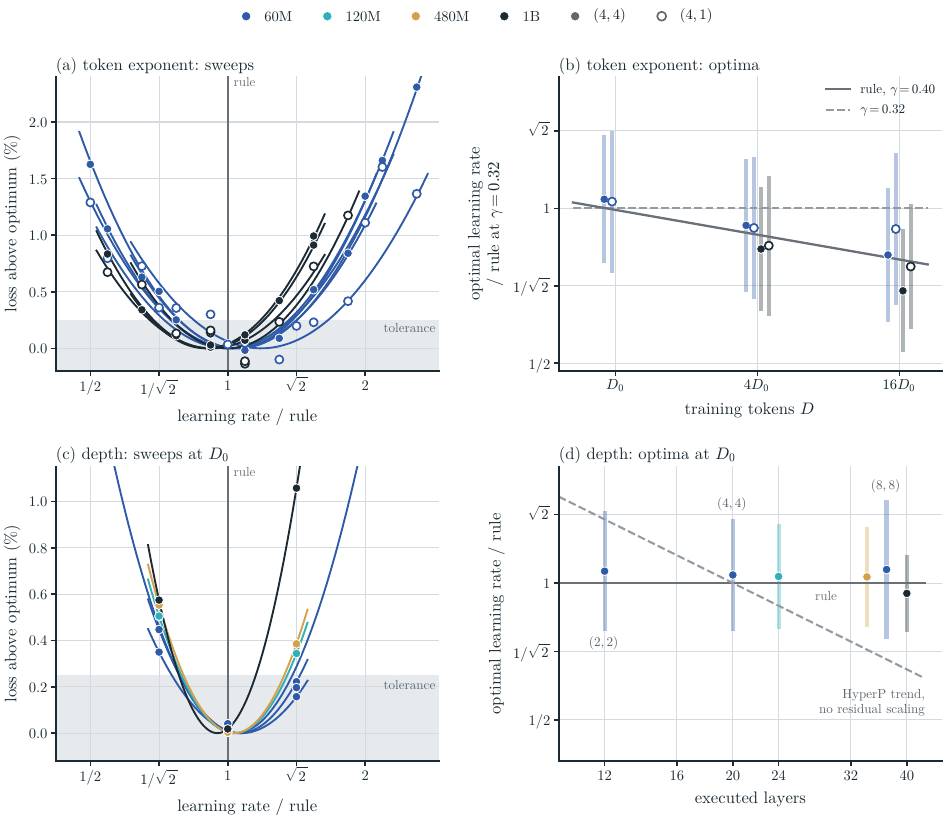}
  \caption{\textbf{Learning-rate rule (\eqref{eq:lr-transfer}, $\gamma=0.40$) against the measured optimum.} Colors give the model
  size; filled markers: $(4,4)$, open markers: $(4,1)$. Solid gray lines: the rule; dashed lines: the alternative it is tested
  against; shading: the tolerance (0.25\%). (a) Loss above the fitted optimum, in percent of the optimum loss,
  against the learning rate relative to the rule, for the ten token-exponent sweeps, with one quadratic in $\log\eta$ per sweep.
  (b) Fitted optima relative to \eqref{eq:lr-transfer} at $\gamma=0.32$, with the learning rates within the tolerance
  (bars); dashed: $\gamma=0.32$. (c) As (a), at $D_0$, for four sizes with $(4,4)$ and for 60M with $(2,2)$ and $(8,8)$. (d)
  Optima of these sweeps against the number of executed layers; dashed: the trend HyperP measures without residual scaling (the
  optimum halves from 8 to 24 layers), drawn through the rule at 20 layers.}
  \label{fig:app-lr}
\end{figure}

\subsection{Learning-rate rule}
\label{app:lrbs-lr}
We take the transfer rule of the reference recipe \citep{schwethelm2026how} for the peak MuonH learning rate of the hidden
weight matrices,
\begin{equation}
  \eta \;=\; 0.014\,\Bigl(\frac{b}{128}\Bigr)^{0.558}\Bigl(\frac{D}{D_{\mathrm{ref}}}\Bigr)^{-\gamma},
  \qquad D_{\mathrm{ref}}=9.83\times10^{8},
  \label{eq:lr-transfer}
\end{equation}
where $b$ is the global batch in sequences of 2{,}048 tokens; the AdamW learning rates scale with width, and those of the
embedding and output also with batch. The reference recipe uses the token exponent $\gamma=0.32$ of HyperP
\citep{ren2026rethinking}, a transfer rule for MuonH. On our models this exponent is too small: relative to \eqref{eq:lr-transfer} with $\gamma=0.32$,
the optimal learning rate keeps falling as $D$ grows (Figure~\ref{fig:app-lr}b). We thus re-estimate $\gamma$ and, as
our models vary in depth, test for a depth term.

\paragraph{Token exponent: $\gamma=0.40$.}
We swept the learning rate at the smallest and the largest model size, 60M at $D_0$, $4D_0$ and $16D_0$ and 1B at $4D_0$ and
$16D_0$, each with $(4,4)$ and $(4,1)$, and fit a quadratic in $\log\eta$ to each sweep to locate its optimum
(Figure~\ref{fig:app-lr}a). The exponent that minimizes the rule's largest loss above the optimum is $\gamma=0.406$ (95\%
interval 0.400--0.412), and we use $\gamma=0.40$. The rule is then within 0.06\% of the optimum in every sweep, the largest
gap being 60M with $(4,1)$ at $16D_0$, whereas $\gamma=0.32$ ends up 0.45\% above the optimum at 1B and $16D_0$.

\paragraph{No depth correction.}
Depth-wise transfer rules lower the learning rate as networks get deeper
\citep{yang2023tensor,bordelon2023depthwise,dey2025dont,ren2026rethinking}. For updates normalized as in MuonH, keeping the
change of the network per step fixed requires $\eta\propto\ell^{-1/2}$ with residual branches scaled by $\ell^{-1/2}$, or
$\eta\propto\ell^{-1}$ without that scaling, where $\ell$ is the number of executed layers \citep[Theorem~3]{ren2026rethinking}.
Our recipe scales neither the residual branches nor the learning rate with depth, although our models get deeper both with
size (20 to 40 executed layers from 60M to 1B at $T=4$) and with the number of forward loops (12 to 36 at 60M). Yet the
optimum does not follow the depth (Figure~\ref{fig:app-lr}c--d): at $D_0$, the rule's learning rate is the best of $1/\sqrt2$,
1 and $\sqrt2$ times itself in all six sweeps, and the fitted optima stay within about 7\% of the rule, where the trend HyperP
measures without residual scaling would lower them by more than half. The optimal learning rate follows the training tokens,
not the depth, so the rule needs no depth term.

\subsection{Batch-size rule}
\label{app:lrbs-bs}
We set the global batch to
\begin{equation}
  b(D) \;=\; 64\sqrt{D/D_0}\ \text{sequences},
  \label{eq:bs-rule}
\end{equation}
rounded to a ladder whose neighboring steps differ by about $\sqrt2$ ($\dots$, 32, 48, 64, 88, 128, 176, 256, 352). The
loss-optimal batch grows about as $\sqrt D$, as in prior work \citep{li2025predictable,bergsma2025power,filatov2025optimal},
and in eleven batch sweeps with $(4,4)$ (60M at five budgets, 240M and 1B at three) the rule's batch is acceptable in every
one (Figure~\ref{fig:app-bs}a).

\paragraph{Prefactor.}
The loss is flat over about a factor of two in batch, so the prefactor is a choice between neighboring ladder steps. One step
lower, $45\sqrt{D/D_0}$, suits 1B with full \ac{BPTT} as well as 64 suits 60M, but truncated \ac{BPTT} prefers
larger batches: at 1B and $16D_0$, the best batch is 176 for $(4,4)$ and 352 for $(4,1)$. Over all 33 sweeps with at least three
batch sizes, which cover all ten core settings, only the rule is acceptable everywhere: it stays within 0.17\% of the best
batch, while one step lower reaches 0.30\% and one step higher 0.42\% (Figure~\ref{fig:app-bs}b). The rule therefore changes
comparisons between loop settings by less than 0.004 nats.

\begin{figure}[t]
  \centering
  \includegraphics[width=\textwidth]{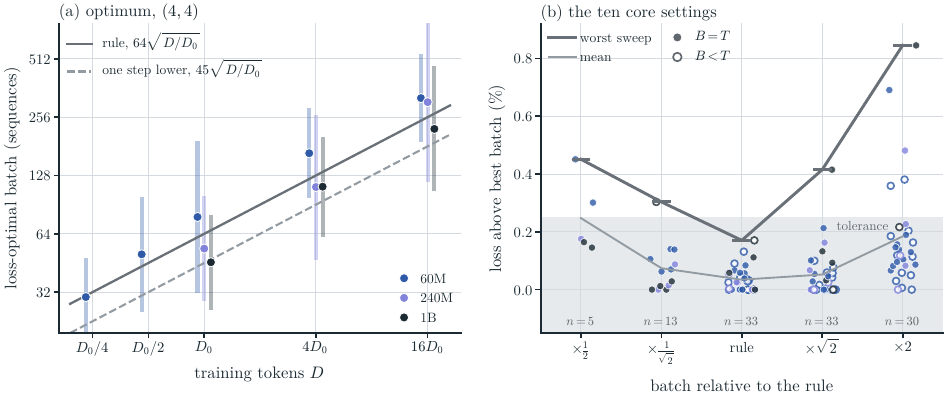}
  \caption{\textbf{Batch rule \eqref{eq:bs-rule} against the measured optimum.} Colors give the model size. (a)
  Loss-optimal batch for $(4,4)$, the vertex of a quadratic in log batch fitted to each sweep, with the batches within the
  tolerance (bars). Lines: the rule (solid) and one ladder step lower (dashed). (b) Loss above the best batch, in percent of
  the best loss, at up to two ladder steps from the rule, over the 33 sweeps with at least three batch sizes (filled: $B=T$;
  open: $B<T$; $n$: sweeps run at that batch). Thick line: worst sweep; thin line: mean; shading: the tolerance.}
  \label{fig:app-bs}
\end{figure}

\section{The law across the five-size grid}
\label{app:grid}
\begin{figure}[!h]
  \centering
  \includegraphics[width=\textwidth]{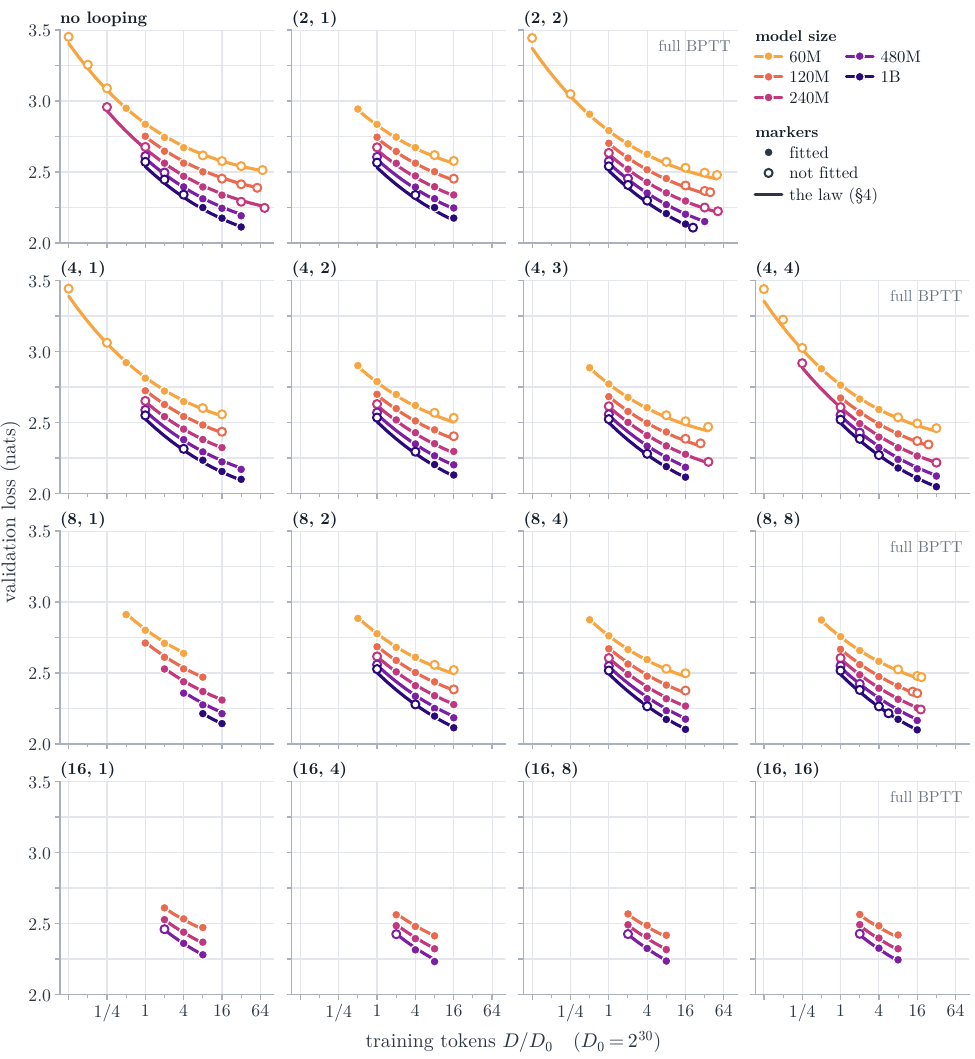}
  \caption{\textbf{The law across the scaling suite.} Validation loss of every configuration in the five-size grid against
  training tokens ($D_0=2^{30}$), one panel per loop setting $(T,B)$ with a row per forward-loop count, colored by model size.
  Filled dots: configurations the law is fitted on; open dots: the other configurations; lines: the law (\Cref{sec:law}) over
  each size's measured budgets.}
  \label{fig:app_allruns}
\end{figure}

Figure~\ref{fig:app_allruns} shows the validation loss of the 343 configurations in the suite's five-size grid, one panel per
loop setting. Filled markers denote the 226 configurations used to fit the law of \Cref{sec:law}; open markers denote the rest.
The suite's four fitted configurations from the compute-matched pairs (120M, 79M, 240M and 162M) are listed
in Table~\ref{tab:compute-optimal-pairs} and are not drawn here.
The figure contains 290 configurations in the ten core settings, 17 with $(8,1)$ and 36 with $T=16$.
Each configuration shows the mean loss of its accepted runs (385 in all; \Cref{app:setup}); the law is drawn over each size's
measured budgets. The configurations span 1.1 to 1{,}291 tokens per parameter, while the fitted ones span 9 to 74.
Per loop setting, the law's RMSE on the fitted configurations is 0.0031 to 0.0093 nats; on the rest it is 0.0200, as most
lie below 5 or above 125 tokens per parameter (\Cref{sec:conclusion}). Figure~\ref{fig:scaling} shows the four full-\ac{BPTT} settings at
the 26 size--budget pairs where all four were trained: one set of coefficients tracks all four, with RMSE 0.0053 nats over
the 68 fitted configurations (0.0148 over the 36 others).

\section{Training curves}
\label{app:curves}

Figure~\ref{fig:app_curves} shows the training curves of the 385 runs behind the 343 configurations of \Cref{app:grid}, one
panel per loop setting; Figure~\ref{fig:curves} in \Cref{sec:setup} repeats its first row. Each curve is the training loss of every tenth step, the cross-entropy of that step's batch, against the
training compute spent so far, counted as in \Cref{app:compute-flops} with the optimizer update; it is not the validation loss
on held-out documents that we report (\Cref{app:setup}). Each token budget is a separate run whose cosine schedule ends at that
budget, so a run reaches its final loss only at its end. The highlighted line therefore joins run ends only (dots), those that end below
every run with less compute, and needs no fit. As in the training-curve envelopes of
\citet[Figure~A5]{hoffmann2022training}, larger models take over this frontier as compute grows; at the highest compute of
each panel it follows the largest model of that setting, 1B, or 480M for $T=16$.

\begin{figure}[!h]
  \centering
  \includegraphics[width=\textwidth]{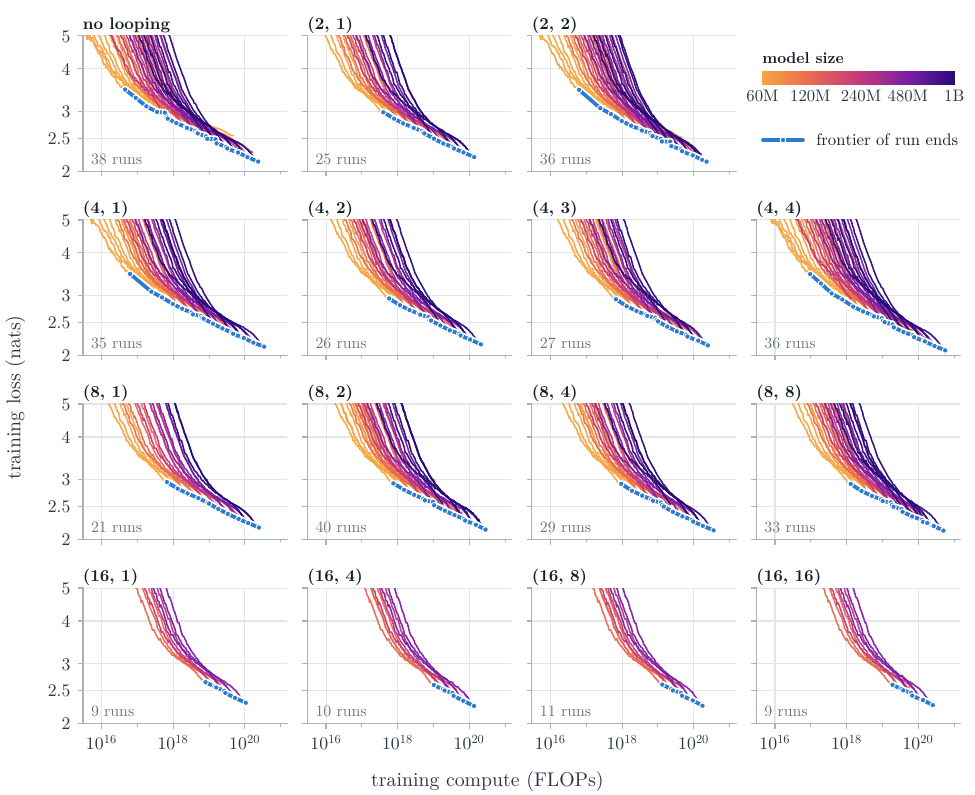}
  \caption{\textbf{Training curves of every run in the five-size grid.} Training loss against compute
  (\Cref{app:compute-flops}), one panel per loop setting, with a row per forward-loop count and backward loops increasing to
  the right, colored by model size; every tenth step, averaged over at most $\pm250$ steps. Highlighted line and dots:
  the frontier of run ends, the runs that end below every run with less compute.}
  \label{fig:app_curves}
\end{figure}

\end{document}